\documentclass{article}

\usepackage[final]{neurips_2024}

\usepackage[utf8]{inputenc} 
\usepackage[T1]{fontenc}    
\usepackage{hyperref}       
\usepackage{url}            
\usepackage{booktabs}       
\usepackage{amsfonts}       
\usepackage{nicefrac}       
\usepackage{microtype}      
\usepackage{xcolor}         

\usepackage{amsmath}
\usepackage{amssymb}
\usepackage{mathtools}
\usepackage{amsthm}
\usepackage{bbm}
\usepackage{bm}
\usepackage{tikz}
\usetikzlibrary{calc}
\usepackage{mathtools}
\usepackage{thmtools}
\usepackage[ruled]{algorithm2e}
\usepackage{booktabs}
\usepackage{float}
\usepackage{enumerate}
\usepackage[textsize=tiny]{todonotes}
\mathtoolsset{showonlyrefs=true}  

\usepackage{minitoc}  

\theoremstyle{plain}
\newtheorem{theorem}{Theorem}[section]
\newtheorem{proposition}[theorem]{Proposition}

\theoremstyle{definition}
\newtheorem{definition}[theorem]{Definition}

\newtheorem{example}[theorem]{Example}
\theoremstyle{remark}
\newtheorem{remark}[theorem]{Remark}

\newcommand{\bc}[1]{\left\{{#1}\right\}}
\newcommand{\br}[1]{\left({#1}\right)}
\newcommand{\bs}[1]{\left[{#1}\right]}
\newcommand{\abs}[1]{\left| {#1} \right|}
\newcommand{\myvec}[1]{\ensuremath{\begin{bmatrix}#1\end{bmatrix}}}
\newcommand{\norm}[1]{\left\lVert#1\right\rVert}
\newcommand{\snorm}[1]{\norm{#1}}
\newcommand\ip[2]{\langle #1, #2 \rangle}

\newcommand{\ones}{\bm 1}
\newcommand{\ind}{\mathbbm{1}}

\newcommand{\ima}{\operatorname{im}}
\newcommand{\defeq}{\vcentcolon=}
\newcommand{\eqdef}{=\vcentcolon}

\DeclareMathOperator{\interior}{\operatorname{int}}
\DeclareMathOperator{\relint}{\operatorname{relint}}
\DeclareMathOperator{\relbd}{\operatorname{relbd}}
\DeclareMathOperator{\dom}{\operatorname{dom}}
\DeclareMathOperator{\rank}{\operatorname{rank}}
\DeclareMathOperator*{\argmin}{\arg\!\min}
\DeclareMathOperator*{\argmax}{\arg\!\max}

\DeclareMathOperator{\spn}{\operatorname{span}}

\DeclareMathOperator{\conv}{\operatorname{conv}}

\newcommand{\R}{\mathbb{R}}
\newcommand{\N}{\mathbb{N}}
\newcommand{\E}{\mathbb{E}}

\newcommand{\Scal}{\mathcal{S}}
\newcommand{\Acal}{\mathcal{A}}
\newcommand{\SAcal}{\Scal\times\Acal}
\newcommand{\Bcal}{\mathcal{B}}

\newcommand{\Dcal}{\mathcal{D}}
\newcommand{\Hcal}{\mathcal{H}}
\newcommand{\Rcal}{\mathcal{R}}
\newcommand{\Mcal}{\mathcal{M}}
\newcommand{\Ecal}{\mathcal{E}}
\newcommand{\Fcal}{\mathcal{F}}
\newcommand{\Ocal}{\mathcal{O}}
\newcommand{\Pcal}{\mathcal{P}}
\newcommand{\Ucal}{\mathcal{U}}
\newcommand{\Vcal}{\mathcal{V}}
\newcommand{\Wcal}{\mathcal{W}}
\newcommand{\Xcal}{\mathcal{X}}
\newcommand{\Ycal}{\mathcal{Y}}

\newcommand{\Kcal}{\mathcal{K}}

\newcommand{\f}[1]{\bm{#1}}

\newcommand{\muhat}{\hat{\mu}}
\newcommand{\rhat}{\hat{r}}

\newcommand{\epsilonhat}{\hat{\varepsilon}}
\newcommand{\deltahat}{\hat{\delta}}

\newcommand{\alphabar}{\bar{\alpha}}

\newcommand{\Rbar}{\overline{\R}}

\newcommand{\epsilonbar}{\bar{\varepsilon}}
\newcommand{\mubar}{\bar{\mu}}

\newcommand{\Pp}{\f{\mathbbm{P}}}
\newcommand{\numin}{\nu_{\text{min}}}
\newcommand{\pimin}{\pi_{\text{min}}}
\newcommand{\piminsh}{\pi_{\text{min, Sh}}}
\newcommand{\pimints}{\pi_{\text{min, Ts}}}

\newcommand{\muE}{\mu^{\textsf{E}}}

\newcommand{\muEhatk}{\hat{\mu}_{\DE_k}}

\newcommand{\rE}{r^{\textsf{E}}}
\newcommand{\DE}{\mathcal{D}^{\textsf{E}}}
\newcommand{\HE}{H^{\textsf{E}}}
\newcommand{\NE}{N^{\textsf{E}}}

\newcommand{\RL}{\mathsf{RL}}
\newcommand{\ARL}{\mathsf{A}}

\newcommand{\IRL}{\mathsf{IRL}}

\newcommand{\polreg}{h}
\newcommand{\occreg}{\bar{h}}

\newcommand{\DR}{D}
\newcommand{\rmax}{R}

\newcommand\numberthis{\addtocounter{equation}{1}\tag{\theequation}}

\title{Towards the Transferability of Rewards Recovered \\via Regularized Inverse Reinforcement Learning}

\author{%
  Andreas Schlaginhaufen \\
  SYCAMORE, EPFL\\
  \texttt{andreas.schlaginhaufen@epfl.ch} \\
  \And
  Maryam Kamgarpour \\
  SYCAMORE, EPFL \\
  \texttt{maryam.kamgarpour@epfl.ch} \\
}

\begin{document}
\doparttoc 
\faketableofcontents 
\part{} 
\maketitle
\begin{abstract}
Inverse reinforcement learning (IRL) aims to infer a reward from expert demonstrations, motivated by the idea that the reward, rather than the policy, is the most succinct and transferable description of a task \citep{ng2000algorithms}. However, the reward corresponding to an optimal policy is not unique, making it unclear if an IRL-learned reward is transferable to new transition laws in the sense that its optimal policy aligns with the optimal policy corresponding to the expert's true reward. Past work has addressed this problem only under the assumption of full access to the expert's policy, guaranteeing transferability when learning from two experts with the same reward but different transition laws that satisfy a specific rank condition \citep{rolland2022identifiability}. In this work, we show that the conditions developed under full access to the expert's policy cannot guarantee transferability in the more practical scenario where we have access only to demonstrations of the expert. Instead of a binary rank condition, we propose principal angles as a more refined measure of similarity and dissimilarity between transition laws. Based on this, we then establish two key results: 1) a sufficient condition for transferability to any transition laws when learning from at least two experts with sufficiently different transition laws, and 2) a sufficient condition for transferability to local changes in the transition law when learning from a single expert. Furthermore, we also provide a probably approximately correct (PAC) algorithm and an end-to-end analysis for learning transferable rewards from demonstrations of multiple experts.\looseness-1
\end{abstract}

\section{Introduction}
Reinforcement learning (RL) has achieved remarkable success in various domains such as robotics \citep{hwangbo2019learning}, autonomous driving \citep{lu2023imitation}, or fine-tuning of large language models \citep{stiennon2020learning}. Despite these advances, a key challenge lies in designing appropriate reward functions that reflect the desired outcomes and align with human values. Misaligned rewards can lead to suboptimal behaviors \citep{ngo2022alignment}, undermining the potential benefits of RL in practical scenarios. Inverse reinforcement learning (IRL), also known as inverse optimal control \citep{kalman1964linear} or structural estimation \citep{rust1994structural}, addresses this problem by inferring a reward from demonstrations of an expert acting optimally in a Markov decision process (MDP).

Compared to behavioral cloning \citep{pomerleau1988alvinn}, which directly fits a policy to the expert's demonstrations, 
IRL is believed to provide a more transferable description of the expert's task \citep{ng2000algorithms}, as recovering the expert’s underlying reward would enable us to train a policy in a new environment with different dynamics. However, it is also known that the reward corresponding to some optimal policy is not unique \citep{ng1999policy}, making it difficult to recover the expert's true underlying reward. This raises the question: \textit{Is a reward recovered via IRL transferable to a new environment in the sense that its optimal policy aligns with the expert's true reward?} For example, in autonomous driving, could we effectively reuse a reward learned from demonstrations of one car in a given city to train or fine-tune a policy for another car in another city?

Ensuring transferability is challenging, as neither the optimal policy corresponding to a reward nor the reward corresponding to an optimal policy is unique. This leads to trivial solutions to the IRL problem, such as constant rewards that make all policies optimal. Common approaches to address this challenge include characterizing the entire set of rewards for which the expert is optimal \citep{metelli2021provably}, or assuming the expert is optimal with respect to an entropy regularized RL problem \citep{ziebart2010modeling}, leading to many popular IRL and imitation learning algorithms \citep{ho2016generative,fu2017,garg2021}. Entropy regularization results in a unique and more uniform optimal policy, serving as a model for the expert's bounded rationality \citep{ortega2015information}. \looseness-1

In the entropy-regularized setting, several recent works study the set of rewards for which a given expert policy is optimal. In particular, \citet{cao2021identifiability, skalse2023invariance} show that under entropy regularization, the expert's reward can be identified up to so-called potential shaping transformations \citep{ng1999policy}. The authors of \citep{pmlr-v202-schlaginhaufen23a} extend this result to more general steep regularization. Furthermore, they show that to guarantee transferability to any transition law, the expert's reward needs to be identified up to a constant. The latter can be achieved either by restricting the reward class, e.g., to state-only rewards \citep{amin2017repeated}, or by learning from multiple experts with the same reward but different transition laws, given that a specific rank condition is satisfied \citep{cao2021identifiability, rolland2022identifiability}. However, the above results cannot be applied directly in practice, as they rely on having full access to the experts' policies, whereas in practice, we typically only have a finite set of demonstrations available.

\paragraph{Contributions} We consider the framework of regularized IRL \citep{jeon2021regularized} and address the transferability of rewards recovered from a finite set of expert demonstrations.
\begin{itemize}
    \item We define a novel notion of transferability (Definition~\ref{def:eps_transferability}), to address the practical limitation of not having perfect access to the experts' policies. Furthermore, we show that when learning from finite data, the conditions developed under full access to the experts' policies are not sufficient to guarantee transferability (Example~\ref{ex1}).
    \item Instead of a binary rank condition, we propose to use principal angles to characterize the similarity and dissimilarity between transition laws (Definition~\ref{def:principal_angles}). 
    Based on these principal angles, we then establish two key transferability results: 1) a guarantee for transferability to any transition laws when learning from at least two experts with sufficiently different transition laws (Theorem~\ref{thm:global_transferability}), and 2) a guarantee for transferability to local changes in the transition law when learning from a single expert (Theorem~\ref{thm:local_transferability}). 
    \item Assuming oracle access to a probably approximately correct (PAC) algorithm for the forward RL problem, we provide a PAC algorithm for the IRL problem, which in $\Ocal(K^2/\epsilonhat^2)$ steps recovers a reward for which, with high probability, all $K$ experts are $\epsilonhat$-optimal (Theorem~\ref{thm:convergence}). Together with our results on transferability, this establishes end-to-end guarantees for learning transferable rewards from a finite set of expert demonstrations. 
    \item We experimentally validate our results in a gridworld environment (Section~\ref{sec:experiments}).\footnote{The code is openly accessible at \url{https://github.com/andrschl/transfer_irl}.}
\end{itemize}
\section{Background}
\paragraph{Notation}
Given $x$ and $y$ in some Euclidean vector space $\Vcal$, we denote the $p$-norm by $\|x\|_p$, the orthogonal projection onto a closed convex set $\Xcal\subset \Vcal$ by $\Pi_{\mathcal{X}}(x) = \arg\min_{y\in\mathcal{X}} \|x-y\|_2$, and the standard dot product by $\ip{x}{y}$. For a linear operator $A$, we denote its image and rank by $\ima A$ and $\rank A$, respectively. Given two sets $\Xcal$ and $\Ycal$, we denote $\Xcal + \Ycal$ for their Minkowski sum and $\Ycal^{\Xcal}$ for the set of all functions mapping from $\Xcal$ to $\Ycal$. Additionally, we denote $\Delta_{\Xcal}$ for the probability simplex over $\Xcal$ and $\ind$ for the indicator function. The interior $\interior\Xcal$, the relative interior $\relint\Xcal$, the relative boundary $\relbd\Xcal$, and the convex hull $\conv\Xcal$ of some set $\Xcal$ are defined in Appendix~\ref{app:sec:notations}, along with an overview of all other notations.
\paragraph{Regularized MDPs}
We consider a regularized MDP \citep{geist2019theory} defined by a tuple $\br{\Scal,\Acal, P, \nu_0, r, \gamma, \polreg}$. Here, $\Scal$ and $\Acal$ represent finite state and action spaces with $|\Scal|, |\Acal|>1$, $\nu_0\in\Delta_{\Scal}$ the initial state distribution, $P\in\Delta_{\Scal}^{\SAcal}$ the transition law, $r\in\R^{\SAcal}$ the reward, and $\gamma\in(0,1)$ the discount factor. Furthermore, $\polreg:\Xcal\to\R$ is a strictly convex regularizer that is defined on a closed convex set $\Xcal\subseteq \R^{\Acal}$ with $\relint \Delta_\Acal\subseteq \interior \Xcal$. Starting from some initial state $s_0\sim\nu_0$ the agent can at each step in time $t$, choose an action $a_t\in\mathcal{A}$, will arrive in state $s_{t+1}\sim P(\cdot|s_t, a_t)$, and receives reward $r(s_t,a_t)$. The goal is to find a Markov policy $\pi\in\Delta_{\Acal}^{\Scal}$ maximizing the regularized objective $\E_{\pi} \bs{\sum_{t=0}^\infty \gamma^t \bs{r(s_t,a_t)-\polreg\br{\pi(\cdot|s_t)}}}$. Following the classical linear programming approach to MDPs \citep{puterman2014markov}, this can be cast equivalently as the convex optimization problem 
\begin{equation}\label{eq:mdp_occ}\tag{O-RL}
    \max_{\mu \in \Mcal} J(r,\mu), \quad \text{with} \quad J(r,\mu)\defeq  \ip{r}{\mu} - \occreg(\mu),
\end{equation}
where $\Mcal$ denotes the set of occupancy measures, $\mu^\pi(s,a)\defeq (1-\gamma)\E_{\pi}\bs{\sum_{t=0}^\infty \gamma^t \ind(s_t=s, a_t=a)}$, and we have $\occreg(\mu)\defeq \E_{(s,a)\sim\mu}\bs{\polreg(\pi^\mu(\cdot|s))}$, with $\pi^\mu$ being the policy corresponding to $\mu$ (see Appendix~\ref{app:sec:notations}). The set of occupancy measures is characterized by the Bellman flow constraints
\begin{equation}\label{eq:bellman_flow}
    \Mcal = \bc{\mu\in\R^{\SAcal}_+ : (E  - \gamma P)^\top\mu = (1-\gamma)\nu_0}\subseteq \Delta_{\SAcal},
\end{equation}
where $E:\R^{\Scal}\to\R^{\Scal\times\Acal}$ and $P:\R^{\Scal}\to\R^{\Scal\times\Acal}$ are the linear operators mapping $v\in\R^{\Scal}$ to $(Ev)(s,a) = v(s)$ and $(Pv)(s,a) = \sum_{s'} P(s'|s,a)v(s')$, respectively. 

Due to the strict convexity of $\polreg$, the regularized MDP problem has a unique optimal policy \citep{geist2019theory}, hence guaranteeing the uniqueness of the optimal occupancy measure in \eqref{eq:mdp_occ}.
In addition, we assume that the gradients of $\polreg$ become unbounded towards the relative boundary of the simplex as detailed in Assumption~\ref{ass:steep_regularization} below.
\vspace{-0.cm}
\begin{restatable}[Steep regularization]{assumption}{steepregularization}
\label{ass:steep_regularization}%
    Suppose that $\polreg:\Xcal\to \R$ is differentiable in $\interior\Xcal$ and that $\lim_{l\to \infty} \norm{\nabla \polreg(p_l)}_2= \infty$ if $\br{p_l}_{l\in\N}$ is a sequence in $\interior\Xcal$ converging to a point $p\in \relbd \Delta_{\Acal}$.
\end{restatable}
Assumption~\ref{ass:steep_regularization} ensures that the optimal policy is non-vanishing, and together with Assumption~\ref{ass:occ_lower_bound} below, we also have that the optimal occupancy measure is non-vanishing.

\vspace{-0.cm}
\begin{restatable}[Exploration]{assumption}{occlowerbound}
\label{ass:occ_lower_bound}%
Let $\nu(s) \defeq \sum_a \mu(s,a)\geq\numin>0$ for any $s\in\Scal$ and $\mu\in\Mcal$.
\end{restatable}
\vspace{-0.cm}
One way to guarantee Assumption~\ref{ass:occ_lower_bound} is to impose a lower bound on the initial state distribution $\nu_0$. 
In the following, it will be convenient to denote the optimal solution to \eqref{eq:mdp_occ} for the reward $r$ as
\begin{equation}\label{eq:opt_occ}
    \RL(r)\defeq \argmax_{\mu\in\Mcal} J(r,\mu),
\end{equation}
and the suboptimality of some occupancy measure $\mu$ for the reward $r$ as
\begin{equation}\label{eq:subopt}
    \mathsf{SubOpt}(r, \mu) \defeq \max_{\mu'\in\Mcal} J(r, \mu') - J(r, \mu).
\end{equation}
That is, $\mu = \RL(r)$ if and only if $\mathsf{SubOpt}(r, \mu) = 0$. 

\begin{remark}
    As we aim to analyze the transferability of rewards to new transition laws $P\in\Delta_{\Scal}^{\SAcal}$, it will often be useful to explicitly specify the dependency on $P$. We do so by adding a subscript -- e.g. we write $\Mcal_P$, $\RL_P$, and $\mathsf{SubOpt}_P$. However, for better readability, we drop these subscripts whenever there is no potential for confusion.
\end{remark}
\vspace{-0.cm}
\paragraph{Inverse reinforcement learning}
Given a dataset of trajectories sampled from an expert $\muE$ that is optimal for some reward $\rE$, the goal in IRL is to recover a reward $\rhat\in\Rcal$, within a predefined reward class $\Rcal\subseteq \R^{\SAcal}$, such that the expert is optimal for $\rhat$. That is, ideally, we aim to find a reward in the feasible reward set 
\begin{equation}\label{eq:feasible_reward_set_def}
    \IRL(\muE) \defeq \bc{r\in\Rcal: \muE\in\RL(r)}.
\end{equation}
However, since we don't have direct access to the expert's policy but only to a finite set of demonstrations, the best we can hope for is an algorithm that with high probability outputs a reward $\rhat\in\Rcal$ such that $\mathsf{SubOpt}(\rhat, \muE)$ is small -- i.e. an algorithm that is PAC \citep{syed2007game}.
\vspace{-0.cm}
\paragraph{Reward equivalence}
The reward corresponding to an optimal occupancy is not unique. For example, all rewards in the affine subspace $r + \Ucal$, where $\Ucal\defeq \ima(E-\gamma P)$ is the subspace of so-called potential shaping transformations, correspond to the same optimal occupancy measure \citep{ng1999policy}. From a geometric perspective, the subspace $\mathcal{U}=\text{im}(E-\gamma P)$ lies perpendicular to the set of occupancy measures $\mathcal{M}$. Therefore, adding an element of $\mathcal{U}$ to the reward leaves the performance difference between any two occupancy measures invariant. Hence, it is often convenient to consider these rewards as equivalent \citep{pmlr-v139-kim21c} and to measure distances between rewards in the resulting quotient space. Given a linear subspace $\Vcal \subset \R^{\SAcal}$, the quotient space $\R^{\SAcal}/\Vcal$ is the set of all equivalence classes $[r]_{\Vcal} \defeq \bc{r'\in \R^{\SAcal} : r' - r \in \Vcal}$, which is itself a vector space with addition and multiplication operation defined by $[r]_{\Vcal} + [r']_{\Vcal} = [r+r']_{\Vcal}$ and $c[r]_{\Vcal} = [cr]_{\Vcal}$ for $c\in\R$. Intuitively, $\mathbb{R}^{\mathcal{S} \times \mathcal{A}}/\mathcal{V}$ is the vector space obtained by collapsing $\mathcal{V}$ to zero, or in other words, it is isomorphic to the orthogonal complement of $\mathcal{V}$. We endow $\R^{\SAcal}/\Vcal$ with the quotient norm $\norm{[r]_{\Vcal}}_2 \defeq \min_{v\in \Vcal}\norm{r+ v}_2 = \norm{\Pi_{\Vcal^\perp}r}_2$ and we say that $r$ and $r'$ are close in $\R^{\SAcal}/\Vcal$ if $\norm{[r]_{\Vcal}-[r']_{\Vcal}}_2$ is small. Moreover, the expert's reward is said to be identifiable up to some equivalence class $[\cdot]_{\Vcal}$ if $\IRL(\muE)\subseteq [\rE]_{\Vcal}$. In this paper, we will consider the equivalence relations induced by constant shifts, i.e., $\Vcal=\ones\defeq\bc{r\in\R^{\SAcal}: r(s,a)=c\in\R}$, and by potential shaping transformations, i.e., $\Vcal=\Ucal$. Note that since $\ones$ is a subspace of $\Ucal$ and $\Ucal$ is $|\Scal|$-dimensional, $[r]_{\ones}$ is a strict subset of $[r]_{\Ucal}$ whenever $|\Scal| > 1$.
\vspace{-0.cm}
\section{Transferability}\label{sec:transferability} 
\vspace{-0.cm}
In this section, we present our main results on transferability in IRL. To this end, we first introduce the problem of learning $\varepsilon$-transferable rewards from multiple experts acting in different environments.
\vspace{-0.cm}
\subsection{Problem formulation}
\vspace{-0.cm}
Let $\Rcal\subseteq\R^{\SAcal}$ be a compact reward class, and suppose we are given access to $K$ expert data sets, 
\small
\begin{equation}
\DE_k=\bc{\br{s_0^{k,i}, a_0^{k,i},\hdots,s_{\HE-1}^{k,i}, a_{\HE-1}^{k,i}}}_{i=0}^{\NE-1},\quad k=0,\hdots, K-1,
\end{equation}
\normalsize
consisting of trajectories sampled independently from the experts $\muE_{P^0}, \hdots, \muE_{P^{K-1}}$. Each expert is optimal for the same unrevealed reward $\rE\in\Rcal$, but under different transition laws, $P^0, \hdots, P^{K-1}$. Our goal is to recover a reward $\rhat\in\Rcal$ that is transferable across a set of transition laws $\Pcal\subseteq\Delta_{\Scal}^{\SAcal}$. Specifically, the optimal occupancy measure corresponding to $\rhat$ should remain approximately optimal for $\rE$ under every transition law in $\Pcal$. This yields the following definition of $\varepsilon$-transferability.
\vspace{-0.cm}
\begin{restatable}[$\varepsilon$-transferability]{definition}{epstransferability}
\label{def:eps_transferability}%
Fix some $\varepsilon>0$. We say that $\rhat$ is $\varepsilon$-transferable to some set of transition laws $\Pcal\subseteq\Delta_{\Scal}^{\SAcal}$ if $\mathsf{SubOpt}_P(\rE, \RL_P(\rhat)) \leq \varepsilon$ for all $P\in\Pcal$. We say that $\rhat$ is exactly transferable to $\Pcal$ if it is $\varepsilon$-transferable to $\Pcal$ with $\varepsilon=0$.
\end{restatable}
\vspace{-0.cm}
The error margin of $\varepsilon$ is crucial, as exact transferability is unrealistic when learning from finite expert data. Moreover, note that Definition~\ref{def:eps_transferability} is a definition of uniform transferability, as it requires $\rhat$ to be $\varepsilon$-transferable to any $P\in\Pcal$ with the same fixed $\varepsilon$.
In the following, we will analyze the transferability of a reward $\rhat$ for which all experts are $\epsilonhat$-optimal for some $\epsilonhat>0$. That is,
\begin{equation}\label{eq:reward_opt}
    \mathsf{SubOpt}_{P^k}(\rhat, \muE_{P^k}) \leq \epsilonhat, \quad k=0,\hdots, K-1.
\end{equation}
In particular, we aim to establish appropriate conditions for choosing $\epsilonhat$ so as to guarantee $\varepsilon$-transferability to some set of transition laws $\Pcal$. In Section~\ref{sec:algorithm}, we will then provide an IRL algorithm that, with high probability, outputs a reward $\rhat$ such that \eqref{eq:reward_opt} holds.
\begin{remark}
As discussed in Appendix~\ref{app:sec:suboptimal_experts}, the assumption of perfect expert optimality with respect to $\rE$ can be relaxed to allow for a misspecification error. All our results remain applicable in this setting but include an additional error term due to the experts' suboptimality.
\end{remark}
\vspace{-0.cm}
\subsection{Related work}
\vspace{-0.cm}
Most previous work has focused on reward identifiability. For a single expert, \citet{cao2021identifiability, skalse2023invariance, pmlr-v202-schlaginhaufen23a} show that under Assumption~\ref{ass:steep_regularization} (steepness) the feasible reward set \eqref{eq:feasible_reward_set_def} can be expressed as
\begin{equation}\label{eq:feasible_reward_set_explicit}
    \IRL(\muE) = \br{\nabla \occreg(\muE) + \Ucal}\cap\Rcal = [\rE]_{\Ucal}\cap\Rcal.
\end{equation}
In other words, steepness ensures that the expert's reward is identifiable up to potential shaping. To identify the reward up to a constant, we can either restrict the reward class, e.g. to state-only rewards as explored by \citet{amin2017repeated}, or learn from multiple experts \citep{cao2021identifiability, rolland2022identifiability}. In particular, when we are given access to two experts, $\muE_{P^0}$ and $\muE_{P^1}$, we can identify the experts' reward up to the intersection
\begin{equation}
    \IRL_{P^0}(\muE_{P^0}) \cap  \IRL_{P^1}(\muE_{P^1}) = [\rE]_{\Ucal_{P^0}}\cap [\rE]_{\Ucal_{P^1}}\cap \Rcal = \br{\rE + \Ucal_{P^0}\cap\Ucal_{P^1}}\cap\Rcal.
\end{equation}
That is, for the unrestricted reward class, $\Rcal=\R^{\SAcal}$, the reward is identifiable up to a constant if and only if $\Ucal_{P^0}\cap\Ucal_{P^1}=\ones$. Or equivalently, if and only if the rank condition
\begin{equation}\label{eq:rank_cond}
    \rank\br{\myvec{E-\gamma P^0, && E-\gamma P^1}} = 2|\Scal| - 1,
\end{equation}
is satisfied \citep{rolland2022identifiability}. Moreover, \citet{pmlr-v202-schlaginhaufen23a} show that identifying the expert's reward up to a constant is a necessary and sufficient condition for exact transferability to any full-dimensional set $\Pcal\subseteq \Delta_{\Scal}^{\SAcal}$ (a set $\Pcal$ whose interior, with respect to the subspace topology on $\Delta_{\Scal}^{\SAcal}$ \citep{bourbaki1966elements}, is non-empty).
\vspace{-0.1cm}
\paragraph{Limitations}
The above results assume perfect access to the expert's policy, which isn't realistic. In practice, we can only learn a reward for which the experts are approximately optimal. In Example~\ref{ex1} below, we show that under approximate optimality of the experts, the learned reward can perform very poorly in a new environment, even if the rank condition in Equation~\eqref{eq:rank_cond} is satisfied.

\begin{example}\label{ex1}
We consider a two-state, two-action MDP with $\Scal=\Acal=\bc{0,1}$, uniform initial state distribution, discount rate $\gamma = 0.9$, and Shannon entropy regularization $\polreg = - \Hcal$ (see Appendix~\ref{app:sec:regularizers}). Suppose the expert reward is $\rE(s,a) = \ind\{s=1\}$ and consider the transition laws, $P^0$ and $P^1$, defined by $P^0(0|s, a) = 0.75$ and $P^1(0|s,a) = 0.25 + \beta\cdot\ind\bc{s=0,a=0}$ for some $\beta\in[0,0.75]$. Also, consider the two experts $\muE_{P^0}=\RL_{P^0}(\rE)$ and $\muE_{P^1}=\RL_{P^1}(\rE)$, and suppose we recovered the reward $\rhat(s,a) = - \rE$. Then, as detailed in Appendix~\ref{app:sec:ex1}, the following holds: 1) We have $\mathsf{SubOpt}_{P^0}(\rhat, \muE_{P^0})=0$ and $\mathsf{SubOpt}_{P^1}(\rhat, \muE_{P^1})= \Ocal(\beta)$. That is, for small $\beta$, the reward $\rhat$ is a good solution to the IRL problem, as both experts are approximately optimal under $\rhat$. 2) The rank condition \eqref{eq:rank_cond} between $P^0$ and $P^1$ is satisfied for any $\beta>0$. 3) For a new transition law $P$ defined by $P(0|s,a)=\ind\bc{s=1,a=0}$, we have $\mathsf{SubOpt}_{P}(\rE, \RL_P(\rhat))\approx 4.81$, i.e. $\RL_P(\rhat)$ performs poorly under the experts' reward.
\end{example}
\vspace{-0.1cm}
\vspace{-0.cm}
\subsection{Theoretical insights}
\vspace{-0.1cm}
To establish a sufficient condition for $\varepsilon$-transferability, our goal is to bound the suboptimality of an optimal occupancy measure, $\RL(r)$, for some reward $r'$, in terms of reward distances measured in the quotient space $\R^{\SAcal}/\Ucal$. To this end, we first establish the relationship between the suboptimality in Equation~\eqref{eq:subopt} and the Bregman divergence corresponding to the occupancy measure regularization.
\vspace{-0.1cm}
\paragraph{Bregman divergences}
The Bregman divergence \citep{teboulle1992entropic} associated to $\occreg$ is defined as
\begin{equation}\label{eq:bregman}
    D_{\occreg}(\mu, \mu') = \occreg(\mu) - \occreg(\mu') - \ip{\nabla \occreg(\mu')}{\mu - \mu'}.
\end{equation}
\vspace{-0.1cm}
\begin{restatable}{proposition}{bregmandivergence}
    \label{prop:bregman_divergence}%
    Under Assumptions~\ref{ass:steep_regularization} and \ref{ass:occ_lower_bound}, we have $\mathsf{SubOpt}(r', \mu) = D_{\occreg}(\mu, \RL(r'))$ for any $\mu\in\Mcal$.
\end{restatable}
\vspace{-0.cm}
Proposition~\ref{prop:bregman_divergence} above demonstrates that the suboptimality of an occupancy measure $\mu$ for the reward $r'$ coincides with the Bregman divergence between $\mu$ and the optimal occupancy measure under $r'$. This generalizes \citep[Lemma 26]{mei2020global} from entropy regularization to any steeply regularized MDP. The proof is presented in Appendix~\ref{app:sec:suboptimality_bounds}.

\vspace{-0.cm}
\paragraph{Reward approximation}
Next, we show that under strong convexity and local Lipschitz gradients, the Bregman divergence between two optimal occupancy measures is bounded in terms of reward distances in $\R^{\SAcal}/\Ucal$. 
\vspace{-0.cm}
\begin{restatable}[Regularity]{assumption}{regularity}
\label{ass:regularity}%
Suppose the following holds:
\vspace{-0.cm}
\begin{enumerate}[a)]
    \item The regularizer $\occreg$ is $\eta$-strongly convex over the set of occupancy measures $\Mcal$. That is, we have
        \begin{equation}\label{eq:strong_convexity}
            \occreg(\mu') \geq \occreg(\mu) + \ip{\nabla\occreg(\mu)}{\mu'-\mu} + \dfrac{\eta}{2}\norm{\mu'-\mu}_2^2,\quad \forall \mu, \mu'\in\Mcal.
        \end{equation}
    \vspace{-0.cm}
    \item The gradient $\nabla\occreg$ is locally Lipschitz continuous over $\relint\Mcal$. That is, for any closed convex subset $\mathcal{K}\subset\relint\Mcal$ there exists $L_{\mathcal{K}}>0$ such that
\begin{equation}\label{eq:lipschitz_gradients}
    \norm{\nabla\occreg(\mu)-\nabla\occreg(\mu')}_2 \leq L_{\mathcal{K}} \norm{\mu -\mu'}_2, \quad \forall \mu, \mu' \in\mathcal{K}.
\end{equation}
\end{enumerate}

\end{restatable}
\vspace{-0.cm}
We will show later that Assumption~\ref{ass:regularity} is met for Shannon and Tsallis entropy regularization (see Proposition~\ref{prop:strong_convexity_lipschitz_gradients}). Under the above assumption, the following lemma establishes the desired upper and lower bound on the Bregman divergence between two optimal occupancy measures with respect to reward distances measured in $\R^{\SAcal}/\Ucal$.
\vspace{-0.cm}
\begin{restatable}{lemma}{bregmanbound}
    \label{lem:bregman_bound}%
    Suppose Assumptions~\ref{ass:steep_regularization},\ref{ass:occ_lower_bound}, and \ref{ass:regularity} hold, and let $r,r'\in\Rcal$. Then, we have
    \begin{equation}\label{eq:bregman_bound}
        \frac{\sigma_{\Rcal}}{2}\norm{[r]_{\Ucal}- [r']_{\Ucal}}_{2}^2 \leq \mathsf{SubOpt}(r', \RL(r)) = D_{\occreg}\br{\RL(r), \RL(r')} \leq \frac{1}{2\eta}\norm{[r]_{\Ucal}- [r']_{\Ucal}}^2_{2},
    \end{equation}
    for some problem-dependent constant $\sigma_{\Rcal}>0$.
\end{restatable}
\vspace{-0.0cm}
\begin{remark}
    The proof of Lemma~\ref{lem:bregman_bound} hinges on the duality between equivalence classes of rewards and optimal occupancy measures (see Appendix~\ref{app:sec:convex_background}). The main idea is to leverage duality of Bregman divergences, and a dual smoothness and strong convexity result in Proposition~\ref{prop:dual_strong_convexity_smoothness}. A key challenge arises because, by Assumption~\ref{ass:steep_regularization}, the regularizer cannot be globally smooth. This results in a problem-dependent dual strong convexity constant $\sigma_{\Rcal}$ \citep{goebel2008local}. In Proposition~\ref{prop:strong_convexity_lipschitz_gradients}, we will provide a lower bound on $\sigma_{\Rcal}$ for the specific choices of Shannon and Tsallis entropy regularization. For more details, we refer to the full proof in Appendix~\ref{app:sec:suboptimality_bounds}.
\end{remark}
\vspace{-0.cm}
The above lemma has two key implications: First, the lower bound in \eqref{eq:bregman_bound} implies that if we recover a reward $\rhat$ for which all experts are approximately optimal, then the distance between $\rhat$ and $\rE$ can be bounded in the quotient spaces $\R^{\SAcal}/\Ucal_{P^k}$. Second, the upper bound shows that to control the performance of $\RL_P(\rhat)$ in a new environment $P$, we need to tightly bound the distance between $\rhat$ and $\rE$ in $\R^{\SAcal}/\Ucal_{P}$. As distances in $\R^{\SAcal}/\Ucal_{P}$ are bounded by distances in $\R^{\SAcal}/\ones$, this can be achieved by bounding the distance between $\rhat$ and $\rE$ in $\R^{\SAcal}/\ones$. However, revisiting Example~\ref{ex1} in light of Lemma~\ref{lem:bregman_bound} shows that even though $\rhat$ and $\rE$ are close in $\R^{\SAcal}/\Ucal_{P^k}$, this does not guarantee their proximity in $\R^{\SAcal}/\ones$ and $\R^{\SAcal}/\Ucal_{P}$.
\vspace{-0.cm}
\renewcommand{\thmcontinues}[1]{continued}
\begin{example}[continues=ex1]
    Recall the definition $\Ucal_{P^k}=\ima(E-\gamma P^k)$. Given that in Example~\ref{ex1} we have $\mathsf{SubOpt}_{P^0}(\rhat, \muE_{P^0})=0$ and $\mathsf{SubOpt}_{P^1}(\rhat, \muE_{P^1})= \Ocal(\beta)$, Lemma~\ref{lem:bregman_bound} ensures that $\rhat$ and $\rE$ coincide in $\R^{\SAcal}/\Ucal_{P^{0}}$, and for small $\beta$, they are close in $\R^{\SAcal}/\Ucal_{P^{1}}$. However, as illustrated in Figure\ref{fig:reward_spaces}(a) this doesn't ensure that $\rhat$ and $\rE$ are close in $\R^{\SAcal}/\ones$ and $\R^{\SAcal}/\Ucal_{P}$. In particular, it can be computed that $\norm{[\rhat]_{\Ucal_P}- [\rE]_{\Ucal_P}}_2\approx 1.51$, which by Lemma~\ref{lem:bregman_bound} explains the poor transferability to $P$.\looseness-1
\end{example}
\subsection{Sufficient conditions for transferability}
\vspace{-0.cm}
With Lemma~\ref{lem:bregman_bound} in place, we are set to present our results on $\varepsilon$-transferability. Example~\ref{ex1} indicates that a sufficient condition for learning transferable rewards from experts ($K=2$) should not rely solely on the binary rank condition \eqref{eq:rank_cond}, which only checks if $\Ucal_{P^{0}}\cap\Ucal_{P^{1}}=\ones$. Instead, we should consider the relative orientation between $\Ucal_{P^0}$ and $\Ucal_{P^1}$. To formalize this, we need to introduce the concept of principal angles between linear subspaces, as outlined in Definition~\ref{def:principal_angles} below.
\vspace{-0.cm}
\begin{restatable}[Principal angles \citep{galantai2013projectors}]{definition}{principalangles}
\label{def:principal_angles}%
Let $\Vcal,\Wcal\subseteq \R^n$ be two subspaces of dimension $m\leq n$. The principal angles $0 \leq \theta_1(\Vcal,\Wcal) \leq \hdots \leq \theta_m(\Vcal,\Wcal) \eqdef \theta_{\max}(\Vcal,\Wcal) \leq\pi/2$ between $\Vcal$ and $\Wcal$ are defined recursively via
\begin{align*}
    \cos(\theta_i(\Vcal,\Wcal)) = &\max_{v\in \Vcal, w\in \Wcal} \ip{v}{w}\; \text{s.t.} \; \norm{v}_2 = \norm{w}_2 = 1, \, \ip{v}{v_j}=\ip{w}{w_j}=0, \,j=1,\hdots, i-1,
\end{align*}
where $v_j, w_j$ are the maximizers corresponding to the angle $\theta_j$. For two transition laws $P,P'$, we define $\theta_i(P, P')\defeq \theta_i(\Ucal_{P}, \Ucal_{P'})$ and refer to $\theta_i(P, P')$ as the $i$-th principal angles between $P$ and $P'$.\looseness-1
\end{restatable}
\vspace{-0.cm}
Principal angles are the natural generalization of angles between two lines or planes to higher dimensional subspaces. For principal angles between transition laws, we have the following proposition.
\vspace{-0.cm}
\begin{restatable}{proposition}{maxangle}
    \label{prop:maxangle_bound}
    Let $P,P'\in\Delta^{\SAcal}_{\Scal}$ and $H_\gamma = 1/(1-\gamma)$. Then, we have $\theta_1(P,P') = 0$ and $\sin\br{\theta_{\max}(P,{P'})}\leq \gamma  H_\gamma\sqrt{|\Scal|/|\Acal|} \snorm{P-P'}$, where $\snorm{\cdot}$ denotes the spectral norm.
\end{restatable}
\vspace{-0.1cm}
The proof can be found in Appendix~\ref{app:sec:perturbation_bounds}. The above result shows that while the first principal angle between two transition laws is always zero, all principal angles are small if the transition laws are close to one another. In Example~\ref{ex1}, we have $\sin(\theta_2(P^0, P^1))=\mathcal{O}(\beta)$, indicating that the second and in this case maximal principal angle is small when $\beta$ is small (see Appendix~\ref{app:sec:ex1}). The following result shows that when learning from two experts, the transferability error is directly controlled by the second principal angle between the experts' transition laws.
\vspace{-0.cm}
\begin{restatable}{theorem}{globaltransferability}
\label{thm:global_transferability}%
Let $K=2$, $\theta_2({P^0}, {P^1})>0$, and suppose that Assumptions~\ref{ass:steep_regularization},\ref{ass:occ_lower_bound}, and \ref{ass:regularity} hold. If $\mathsf{SubOpt}_{P^k}(\rhat, \muE_{P^k}) \leq \epsilonhat$ for $k=0,1$, then $\rhat$ is $\varepsilon$-transferable to $\Pcal=\Delta_{\Scal}^{\SAcal}$ with
\begin{equation*}
    \varepsilon = \epsilonhat /\bs{\eta \sigma_{\Rcal}\sin\br{\theta_2({P^0}, {P^1}}/2)^2}.
\end{equation*}
\end{restatable}
\vspace{-0.cm}
\begin{proof}[Sketch of proof]
The main idea of the proof is illustrated in Figure~\ref{fig:reward_spaces}(b). First, it follows from Lemma~\ref{lem:bregman_bound} that $\rhat$ and $\rE$ are $\epsilonbar = \sqrt{2\epsilonhat/\sigma_{\Rcal}}$-close in $\R^{\SAcal}/\Ucal_{P^k}$ for $k=0,1$, respectively. From Figure~\ref{fig:reward_spaces}(b) we see -- using basic trigonometry -- that this implies that $\rhat$ and $\rE$ are at least $\Delta = \epsilonbar/\sin(\theta/2)$-close in $\R^{\SAcal}/\ones$. As shown in the full proof in Appendix~\ref{app:sec:global_transferability}, the relevant angle, $\theta$, is the second principal angle $\theta_2({P^0}, {P^1})$. The result then follows from the upper bound in Lemma~\ref{lem:bregman_bound}. 
\end{proof}
\vspace{-0.cm}
\begin{figure}[t]
\vspace{-0.cm}
  \centering
  \includegraphics[width=\linewidth]{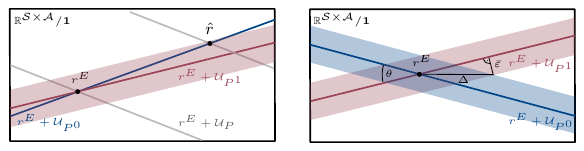}
  \vspace{-0.cm}
    \begin{flushleft}\hspace{1.2cm}(\textit{a}) Rewards in Example~\ref{ex1}.  \hspace{3.1cm}(\textit{b}) Proof sketch Theorem~\ref{thm:global_transferability}.\end{flushleft}\vspace{-0.1cm}
  \caption{(\textit{a}) illustrates the equivalence classes $[\rhat]_{\Ucal}$ and $[\rE]_{\Ucal}$, corresponding to the transition laws $P^0, P^1, P$ from Example~\ref{ex1}, for a small $\beta$, in $\R^{\SAcal}/\ones$. The blue lines correspond to $P^0$, the red lines to $P^1$, and the gray lines to $P$. Furthermore, the shaded areas illustrate the approximation error around $[\rE]_{\Ucal_{P^k}}$, as guaranteed by Lemma~\ref{lem:bregman_bound}. (\textit{b}) illustrates the uncertainty set for the recovered reward when learning from two experts, as discussed in the proof sketch of Theorem~\ref{thm:global_transferability}.
  \label{fig:reward_spaces}}
\end{figure}
\vspace{-0.cm}
Some observations are in order. First, the above theorem shows that the larger the second principal angle between the two experts' transition laws, the better the recovered reward transfers to a new environment. Second, observe that $\theta_2({P^0}, {P^1})>0$ is equivalent to the rank condition \eqref{eq:rank_cond}, as the second principal angle between two subspaces is non-zero if and only if their intersection is at most one-dimensional. 
Therefore, for exact transferability, Theorem~\ref{thm:global_transferability} requires the rank condition \eqref{eq:rank_cond} to be satisfied and $\epsilonhat=0$, recovering the results by \citet{cao2021identifiability, rolland2022identifiability, pmlr-v202-schlaginhaufen23a}. But in contrast to past results, Theorem~\ref{thm:global_transferability} applies to more realistic scenarios, where $\epsilonhat$ is merely small, not zero. Finally, we note that Theorem~\ref{thm:global_transferability} can be trivially generalized to $K\geq2$ experts by replacing $\theta_2({P^0}, {P^1})$ with the maximum of $\theta_2({P^k}, {P^{l}})$ over $0\leq k\leq l\leq K-1$. However, such bounds may be loose for $K>2$, potentially leaving considerable room for improvement in this setting.
\vspace{-0.cm}
\paragraph{Local transferability}
When learning a reward $\rhat$ from a single expert ($K=1$), \citet{pmlr-v202-schlaginhaufen23a} show that, without reducing the dimension of the reward class, $\rhat$ cannot be exactly transferable to any neighborhood of the expert's transition law $P_0$. However, Theorem~\ref{thm:local_transferability} below shows that by allowing for an $\varepsilon$ of error, we can guarantee transferability to a neighborhood of $P_0$.\looseness-1
\begin{restatable}{theorem}{localtransferability}
\label{thm:local_transferability}%
Let $K=1$, $D\defeq \max_{r,r'\in\Rcal} \norm{r-r'}_2$, and suppose that Assumptions~\ref{ass:steep_regularization},\ref{ass:occ_lower_bound}, and \ref{ass:regularity} hold. If $\mathsf{SubOpt}_{P^0}(\rhat, \muE) \leq \epsilonhat$, then $\rhat$ is $\varepsilon_P$-transferable to $P\in\Delta_{\Scal}^{\SAcal}$ with
\begin{equation*}
    \varepsilon_P = 2\max\bc{2\epsilonhat/\sigma_{\Rcal},  D^2\sin\br{\theta_{\max}({P^0}, P)}^2} / \eta.  
\end{equation*}
\end{restatable}
The above theorem (which is proven in Appendix~\ref{app:sec:local_transferability}) shows that the reward learned from a single expert transfers to transition laws that are sufficiently close to the expert's, where the closeness is measured in terms of the maximal principal angle. In other words, while a large second principal angle between two experts' transition laws, as per Theorem~\ref{thm:global_transferability}, ensures that the reward recovered from these two experts is transferable to arbitrary transition laws, a small largest principal angle between two transition laws ensures that a reward recovered in one environment can be successfully transferred to the other environment. 
\begin{remark}
    As discussed in Appendix~\ref{app:sec:estimating_principal_angles}, we can compute the principal angles using a singular value decomposition. Moreover, given estimates $\hat P^0, \hat P^1$ of the transition laws $P^0, P^1$, the error in the estimate of $\sin  \theta_i(P^0, P^1)$ scales with $\Ocal(\max\{||P^0-\hat P^0||, ||P^1-\hat P^1||\})$.
\end{remark}
\looseness-1
\vspace{-0.cm}
\paragraph{Regularizers}\label{sec:entropy}
To provide more insights about Theorems~\ref{thm:global_transferability} and \ref{thm:local_transferability}, we provide explicit values for the primal and dual strong convexity constants, $\eta$ and $\sigma_{\Rcal}$, respectively. To this end, we focus on the Shannon entropy regularization $\polreg(p) = - \tau \Hcal(p)$ and the Tsallis-1/2 entropy regularization $\polreg(p) = - \tau \Hcal_{1/2}(p)$ as defined in Appendix~\ref{app:sec:regularizers}. While the Shannon entropy regularization is commonly used in IRL \citep{ziebart2010modeling, ho2016generative}, the Tsallis-1/2 entropy is more often adopted in the multi-armed bandit literature \citet{zimmert2021tsallis}. Both regularizations satisfy Assumption~\ref{ass:steep_regularization} as well as Assumption~\ref{ass:regularity} with the constants detailed in Proposition~\ref{prop:strong_convexity_lipschitz_gradients} in the appendix.
In general, the Tsallis entropy leads to a slightly smaller strong convexity constant $\eta$, but avoids an exponential dependence on the effective horizon $H_\gamma = 1/(1-\gamma)$ in $\sigma_{\Rcal}$. Below, we summarize the implications of Proposition~\ref{prop:strong_convexity_lipschitz_gradients} for $\varepsilon$-transferability of a reward $\rhat$ recovered from two experts.
\begin{restatable}{corollary}{explicit_transferability}
\label{cor:explicit_transferability}%
Suppose the conditions in Theorem~\ref{thm:global_transferability} hold. Furthermore, let $H_\gamma \defeq 1/(1-\gamma)$, $\rmax\defeq\max_{r\in\Rcal}\norm{r}_\infty$, $\DR = \max_{r,r'\in\Rcal} \norm{r-r'}_2$, and $\tau < D$. Then, for the Shannon entropy $\rhat$ is $\varepsilon$-transferable to $\Pcal=\Delta_{\Scal}^{\SAcal}$ with
\small
\begin{equation}
    \varepsilon = \dfrac{ 2H_\gamma^2 D|\Scal||\Acal|^{2+H_\gamma}\exp\br{\frac{2R H_\gamma}{\tau}}}{\numin^2 \tau \sin\br{\theta_2({P^0}, {P^1})/2}^2} \epsilonhat,
\end{equation}
\normalsize
and for the Tsallis entropy with
\small
\begin{equation}
    \varepsilon = \dfrac{4\sqrt{2}H_\gamma^5 D|\Scal||\Acal|^2\br{2\rmax /\tau + 3\sqrt{|\Acal|}}^3 }{\numin^2 \tau \sin\br{\theta_2({P^0}, {P^1})/2}^2} \epsilonhat.
\end{equation}
\normalsize
\end{restatable}
We observe that transferability generally becomes more challenging with decreasing regularization parameter $\tau$, i.e. if the expert's policy becomes more deterministic. Furthermore, we see that it is easier to recover a transferable reward in a Tsallis entropy-regularized MDP. Corollary~\ref{cor:explicit_transferability} also shows that the constant between $\varepsilon$ and $\epsilonhat$ tends to be large, meaning that we need to recover a reward for which the experts are $\epsilonhat$-optimal with a very small $\epsilonhat$ to guarantee $\varepsilon$-transferability for a reasonable $\varepsilon$. However, it's important to note that our results provide sufficient conditions for the worst case, and it remains for future work to determine under what conditions these constants can be improved.

\begin{remark}
Our results in this section, especially Proposition~\ref{prop:bregman_divergence} and Lemma~\ref{lem:bregman_bound}, are critically relying on the steepness of the regularization (Assumption~\ref{ass:steep_regularization}), which is essential to ensure that the expert's reward can be identified up to the equivalence class of potential shaping transformations. Although we can still upper bound the suboptimality
$\mathsf{SubOpt}(\rE, \RL(\rhat))$ in terms of the distance between $\rhat$ to $\rE$ in $\R^{\SAcal}/\Ucal$ without this assumption (see Proposition~\ref{prop:unreg_subopt_bound}), we no longer have a lower bound as in Lemma~\ref{lem:bregman_bound}, which is essential for establishing closeness of $\rhat$ and $\rE$ in $\R^{\SAcal}/\Ucal$. Hence, we expect it to be difficult to obtain guarantees similar to those in Theorem~\ref{thm:global_transferability} and \ref{thm:local_transferability} for the unregularized setting, without either reducing the dimension of the reward class \citep{amin2017repeated} or making specific assumptions about the feasible reward sets \citep[Assumption 4.1]{metelli2021provably}.\looseness-1
\end{remark}
\vspace{-0.cm}
\section{Algorithm}\label{sec:algorithm}
\vspace{-0.cm}
To provide end-to-end guarantees for recovering transferable rewards from a finite set of expert demonstrations, we analyze the convergence and sample complexity (in terms of expert demonstrations) of an algorithm for recovering a reward for which all $K$ experts are approximately optimal. To this end, we focus on the reward class $\Rcal=\bc{r\in\R^{\SAcal}: \norm{r}_1\leq 1}$. Furthermore, we assume oracle access to a $(\varepsilon, \delta)$-PAC algorithm for the forward problem \eqref{eq:mdp_occ}. That is, a polynomial-time algorithm, $\ARL^{\varepsilon,\delta}$, that outputs a policy $\pi = \ARL^{\varepsilon,\delta}(r)$ such that with probability at least $1-\delta$ it holds that $\mathsf{SubOpt}(r, \mu^\pi)\leq\varepsilon$ (see e.g. \citep{lan2023policy} for a specific example). The key idea of our meta-algorithm is to learn a reward minimizing the sum of the suboptimalities of the $K$ experts $\muE_{P^0}, \hdots, \muE_{P^{K-1}}$. This leads us to the following multi-expert IRL problem
\begin{equation*}\label{eq:irl}\tag{O-IRL}
    \min_{r\in\Rcal} \sum_{k=0}^{K-1}\mathsf{SubOpt}_{P^k}(r, \muEhatk),
\end{equation*}
where $\muEhatk(s,a) \defeq (1-\gamma)/\NE \sum_{i=0}^{\NE-1} \sum_{t=0}^{\HE-1} \gamma^t\ind\{s_t^{k,i}=s, a_t^{k,i}=a\}$ is an empirical expert occupancy measure. To solve Problem \eqref{eq:irl}, we propose the projected gradient descent scheme as detailed in Algorithm~\ref{alg:multi_expert_irl} below, where $\texttt{rollout}_{P^k}(\pi, N, H)$ samples $N$ independent trajectories of length $H$ from policy $\pi$. Using a stochastic online gradient descent analysis, Theorem~\ref{thm:convergence} shows that any PAC algorithm for the forward problem yields a PAC algorithm for the inverse problem.
\vspace{-0.cm}
\begin{algorithm}[h]
\SetKwComment{Comment}{\%}{}
\SetKwInput{Input}{Input}
\SetKwInput{Initialize}{Initialize}
\SetKwInput{Return}{Return}
    \caption{Multi-expert IRL\label{alg:multi_expert_irl}}
    \DontPrintSemicolon
    \Input{$\alpha, T, \{\DE_k\}_{k=0}^{K-1}, N, H, \varepsilon_{\text{opt}}, \delta_{\text{opt}}$.}
    \Initialize{$\pi\in\Delta_{\Acal}^{\Scal}$ and $r\in\Rcal$ arbitrarily.}
    \For{$i=0, \hdots, T-2$}{
        \For{$k=0, \hdots, K-1$}{
            $\pi_{k,t} = \ARL_{P^k}^{\varepsilon_{\text{opt}}, \delta_{\text{opt}}}(r_t)$ \tcp*{Forward RL.}
            $\Dcal_{k,t} = \texttt{rollout}_{P^k}(\pi_{k,t}, N, H)$\; 
        }
        $g_t = \sum_{k=0}^{K-1}\br{\muhat_{\Dcal_{k,t}} - \muhat_{\DE_k}}$\;
        
        $r_{t+1} = \Pi_{\Rcal}(r_t - \alpha g_t)$  \tcp*{Reward step.}
    }
    \Return{$\rhat \defeq \frac{1}{T}\sum_{t=0}^{T-1} r_t$.}
\end{algorithm}
\vspace{-0.cm}
\begin{restatable}{theorem}{irlconvergence}
\label{thm:convergence}%
    Suppose that $\NE=\Omega\big(K\log(|\Scal||\Acal|/\deltahat)/\epsilonhat^2\big)$ and $\HE=\Omega\big(\log(K/\epsilonhat)/\log(1/\gamma)\big)$. Running Algorithm~\ref{alg:multi_expert_irl} for $T=\Omega\big(K^2/\epsilonhat^2\big)$ iterations with step-size $\alpha=1/(K\sqrt{T})$, where $\delta_{\text{opt}} = \Ocal\big(\deltahat\epsilonhat^2/K^3\big)$ , $\varepsilon_{\text{opt}}=\Ocal(\epsilonhat/K)$, $N=\Omega\big(K\log(K|\Scal||\Acal|/(\deltahat\epsilonhat))/\epsilonhat^2\big)$, and $H=\HE$, it holds with probability at least $1-\deltahat$ that $\mathsf{SubOpt}_{P^k}(\rhat, \muE_{P^k}) \leq \epsilonhat, \; \text{ for } k=0,\hdots, K-1.$
\end{restatable}
The above result generalizes \citep[Theorem 2]{syed2007game} by considering multiple experts and by proving convergence in terms of the expert suboptimality. We refer to Appendix~\ref{app:sec:irl_convergence} for the proof and the precise constants. Theorem~\ref{thm:convergence} shows that with $\Omega(K/\epsilonhat^2)$ demonstrations of each expert, we recover in $\Omega(K^2/\epsilonhat^2)$ steps of Algorithm~\ref{alg:multi_expert_irl} a reward $\rhat$ for which all experts are $\epsilonhat$-optimal. Together with Theorem~\ref{thm:global_transferability} and \ref{thm:local_transferability}, this provides a bound on the sample and time complexity of recovering in $\varepsilon$-transferable rewards in regularized IRL.\looseness-1
\vspace{-0.cm}
\begin{figure}[t]
  \centering
  \includegraphics[width=\linewidth]{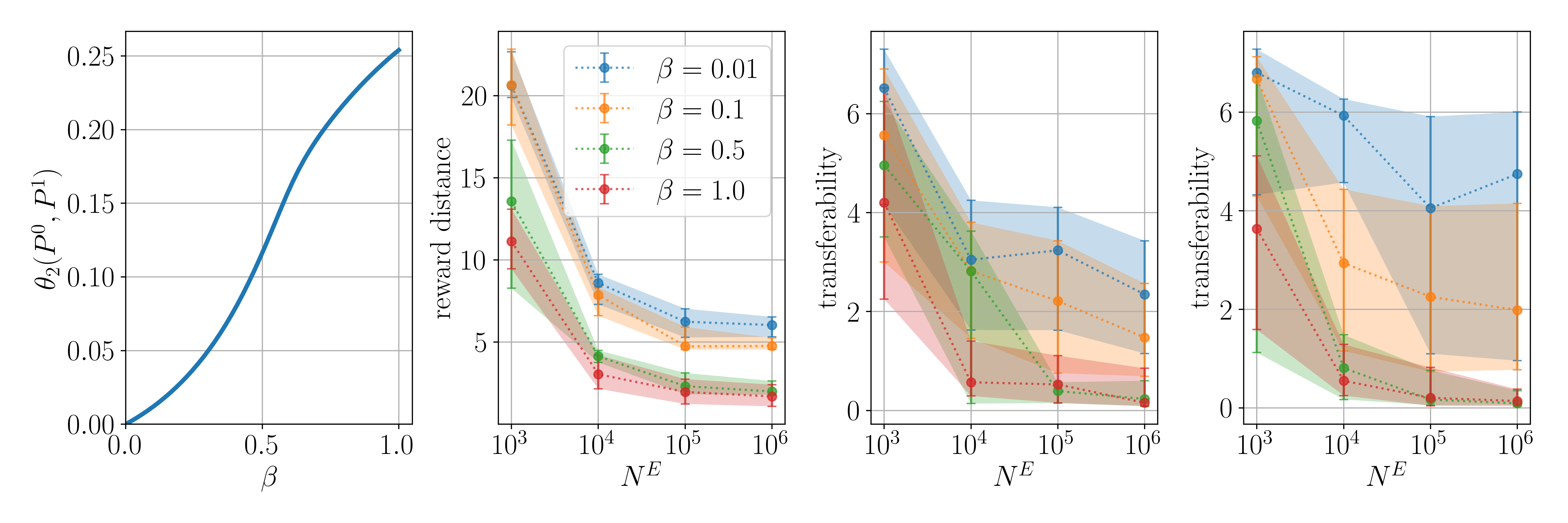}
  \vspace{-0.cm}
    \begin{flushleft}\hspace{2.2cm}(\textit{a})  \hspace{2.8cm}(\textit{b})  \hspace{2.8cm}(\textit{c})  \hspace{2.8cm}(\textit{d})\end{flushleft}\vspace{-0.1cm}
  \caption{(\textit{a}) shows the second principal angle between the experts, for varying wind strength $\beta$. (\textit{b}) shows the distance between $\rhat$ and $\rE$ in $\R^{\SAcal}/\ones$ for a varying number of expert demonstrations $\NE$ and wind strength $\beta$. (\textit{c}) and (\textit{d}) show the transferability to $P^{\text{South}}$ and $P^{\text{Shifted}}$ in terms of $\mathsf{SubOpt}_{P^{\text{South}}}(\rE, \RL_{P^{\text{South}}}(\rhat))$ and $\mathsf{SubOpt}_{P^{\text{Shifted}}}(\rE, \RL_{P^{\text{Shifted}}}(\rhat))$, respectively. The circles indicate the median and the shaded areas the 0.2 and 0.8 quantiles over 10 independent realizations of the expert data.
  \label{fig:experiment_summary}}
  \vspace{-0.cm}
\end{figure}
\vspace{-0.cm}
\section{Experiments}\label{sec:experiments}
To validate our results experimentally, we adopt a stochastic variant of the \texttt{WindyGridworld} environment \citep{sutton2018reinforcement}. In this environment, the agent moves to the intended grid cell with a probability of $(1-\beta)$ and is pushed one step further in the direction of the wind with a probability of $\beta$. 
Using Algorithm 1, we recover a reward $\rhat$ from demonstrations of two experts, both exposed to the same wind strength $\beta$ but different wind directions -- North and East. The experiments are repeated for a varying number of expert demonstrations $\NE\in\{ 10^3, 10^4, 10^5, 10^6\}$ and wind strengths $\beta\in \{0.01,0.1,0.5,1.0 \}$. We then test the transferability to two different environments: one with South wind, $P^{\text{South}}$, and a zero-wind environment with cyclically shifted actions, $P^{\text{Shifted}}$. Figure~\ref{fig:experiment_summary}(a) shows that the second principal angle between the two experts' transition laws $P_0$ and $P_1$ increases with increasing wind strength. Moreover, Figure~\ref{fig:experiment_summary}(b)-(d) show that both the closeness between $\rhat$ and $\rE$ in $\R^{\SAcal}/\ones$ and the transferability to $P^{\text{South}}$ and $P^{\text{Shifted}}$ improve with a larger second principal angle, as expected from Theorem~\ref{thm:global_transferability}. For a more detailed discussion of the experiments we refer to Appendix~\ref{app:sec:experiments}.\looseness-1

\vspace{-0.cm}
\section{Conclusion}\label{sec:conclusion}
\vspace{-0.cm}
\paragraph{Summary} In this paper, we investigated the transferability of rewards in regularized IRL. We showed that the conditions established under full access to the experts' policies do not guarantee transferability when learning a reward from a finite set of expert demonstrations. To address this issue, we proposed using principal angles as a more refined measure of the similarity and dissimilarity of transition laws. Assuming a strongly convex and locally smooth regularization, we then showed that if we recover a reward for which at least two experts are nearly optimal, and their environments are sufficiently different in terms of the second principal angle between their transition laws, then the recovered reward is universally transferable. Furthermore, we showed that if two environments are sufficiently similar in terms of the maximal principal angle between their transition laws, rewards learned in one environment can be effectively transferred to the other environment. Additionally, we provided explicit constants for the Shannon and Tsallis entropy, as well as a PAC algorithm for recovering a reward for which all experts are approximately optimal. As a result, we established end-to-end guarantees for learning transferable rewards in regularized IRL. Additionally, we experimentally validated our results through gridworld experiments.
\vspace{-0.cm}
\paragraph{Limitations and future work}
Our results provide only sufficient conditions for transferability. It would be valuable to investigate necessary conditions to check whether our bounds are tight. Furthermore, extending our analysis to lower-dimensional reward classes could reduce the complexity of learning transferable rewards. Although our paper focuses on discrete state and action spaces, an exciting avenue for future research would be to extend our results to continuous state and action spaces, which are more commonly encountered in practice. We expect that our proof methods can be generalized to this setting, but the analysis will be more intricate due to the infinite-dimensional reward and occupancy measure spaces. Finally, as our work is mainly theoretical, experimental validation on real-world applications could provide valuable insight into the practical aspects and challenges of transferability.
\paragraph{Acknowledgments}
Andreas Schlaginhaufen is funded by a PhD fellowship from
the Swiss Data Science Center.
\bibliographystyle{abbrvnat}
\bibliography{refs}


\newpage
\appendix
\addcontentsline{toc}{section}{Appendix} 
\part{Appendix} 
\parttoc 
\newpage


\section{Notations}\label{app:sec:notations}
\paragraph{Overview}
Here, we provide an overview of some of the most important notations. However, every notation is defined when it is introduced as well.
\begin{table}[H]
\caption{Notations.}
\vspace{-0.3cm}
\label{tab:notations}
\vskip 0.15in
\begin{center}
\begin{tabular}{ll}
\toprule
$\Ycal^{\Xcal}$ & $\defeq \quad$ set of functions $f:\Xcal\to \Ycal$\\
\midrule
$\Delta_{\Xcal}$ & $\defeq \quad$ probability simplex over some discrete set $\Xcal$\\
\midrule
$\Mcal$ & $\defeq \quad$ set of feasible occupancy measures\\
\midrule
$\Rcal$ & $\defeq \quad \bc{r\in\R^{\SAcal}: \norm{r}_1\leq 1}$, reward class\\
\midrule
$R$ & $\defeq \quad \max_{r\in\Rcal} \norm{r}_\infty$, reward bound\\
\midrule
$\DR$ & $\defeq \quad \max_{r,r\in\Rcal} \norm{r-r'}_2$, diameter of the reward class\\
\midrule
$H_\gamma$ & $\defeq \quad 1/(1-\gamma)$, effective horizon\\
\midrule
$\nu(s)$ & $\defeq \quad \sum_{a} \mu(s,a)$\\
\midrule
$\pi^\mu(a|s)$ & $\defeq \quad \begin{cases}
    \mu(s,a) / \sum_{a'}\mu(s,a') = \mu(s,a) / \nu(s) &,\nu(s)>0 \\
    1/|\Acal| \quad \text{(arbitrary)} &, \text{otherwise}
\end{cases}$  \\
\midrule
$\pi_s$ & $\defeq \quad  \pi(\cdot|s)$  \\
\midrule
$\occreg(\mu)$ & $\defeq \quad \E_{(s,a)\sim\mu}\bs{\polreg(\pi_s^\mu)}$ \\
\midrule
$J(r, \mu)$ & $\defeq \quad \ip{r}{\mu} - \occreg(\mu)$\\
\midrule
$\mathsf{SubOpt}(r, \mu')$ & $\defeq \quad \max_{\mu\in\Mcal} J(r, \mu) - J(r, \mu')$\\
\midrule
$\RL(r)$ & $\defeq \quad \argmax_{\mu\in\Mcal} J(r, \mu)$, optimal occupancy measure for $r$\\
\midrule
$\IRL(\muE)$ & $\defeq \quad \bc{r\in\Rcal: \muE = \RL(r)}$, feasible reward set for $\muE$\\
\midrule
$\ARL^{\varepsilon,\delta}(r)$ & $\defeq \quad$ PAC RL algorithm outputting some $\varepsilon$-optimal policy with probability at least $1-\delta$\\
\midrule
$\muhat_{\Dcal}$ & $\defeq \quad \frac{1-\gamma}{N} \sum_{i=0}^{N-1} \sum_{t=0}^{T-1} \gamma^t \ind\bc{s_t^i = s, a_t^i = a}$, where $\Dcal = \bc{\br{s_0^i, a_0^i, \hdots, s_{H-1}^i, a_{H-1}^i}}_{i=0}^{N-1}$\\
\midrule
$E$ & $\defeq \quad$ linear operator $\R^{\Scal}\to \R^{\SAcal}$ defined by $(Ef)(s,a) = f(s)$ for $f\in\R^{\Scal}$\\
\midrule
$P$ & $\defeq \quad$ linear operator $\R^{\Scal}\to \R^{\SAcal}$ defined by $(Pf)(s,a) = \sum_{s'} P(s'|s,a) f(s')$ for $f\in\R^{\Scal}$\\
\midrule
$\ima\br{A}$ & $\defeq \quad$ image of a linear operator $A$\\
\midrule
$\Ucal$ & $\defeq \quad$ $\ima(E-\gamma P) \subset\R^{\SAcal}$, potential shaping subspace\\
\midrule
$\ones$ & $\defeq \quad$ $\bc{f = \text{constant}} \subset\R^{\SAcal}$, constant subspace\\
\midrule
$\Vcal^\perp$ & $\defeq \quad$ orthogonal complement of a linear subspace $\Vcal$\\
\bottomrule
\end{tabular}
\end{center}
\vskip -0.1in
\end{table}

\paragraph{Additional definitions}
In the following, we briefly recall some additional definitions. To this end, we denote $\Bcal(x,r)\defeq \bc{x\in \R^n:\norm{x}_2 <r}$ for an open ball of radius $r$ with center $x$.
\begin{definition}[Interior]
The interior of a set $\mathcal{X}\subseteq\R^n$ is defined as
\begin{equation*}
    \interior\mathcal{X} \defeq \bc{x\in\mathcal{X}: \Bcal(x,r)\subseteq \mathcal{X} \text{ for some } r>0}.
\end{equation*}
\end{definition}
\begin{definition}[Affine hull]
The affine hull of a set $\mathcal{X}\subseteq\R^n$ is defined as
\begin{equation*}
    \operatorname{aff}\mathcal{X} \defeq \bc{\theta_1 x_1 + \hdots + \theta_k x_k:x_1,\hdots,x_k\in\mathcal{X}, \theta_1+ \hdots + \theta_k=1}.
\end{equation*}
\end{definition}
\begin{definition}[Relative interior]
The relative interior of a set $\mathcal{X}\subseteq\R^n$ is defined as
\begin{equation*}
    \relint\mathcal{X} \defeq \bc{x\in\mathcal{X}: \Bcal(x,r)\cap \operatorname{aff}\mathcal{X}\subseteq \mathcal{X} \text{ for some } r>0}.
\end{equation*}
\end{definition}
\begin{definition}[Relative boundary]
The relative boundary of a closed set $\mathcal{X}\subseteq\R^n$ is defined as
\begin{equation*}
    \relbd\mathcal{X} \defeq \mathcal{X} \setminus \relint \mathcal{X}.
\end{equation*}
\end{definition}
\begin{definition}[Convex hull]
The convex hull of a set $\mathcal{X}\subseteq\R^n$ is defined as
\begin{equation*}
    \conv\mathcal{X} \defeq \bc{\theta_1 x_1 + \hdots + \theta_k x_k:x_1,\hdots,x_k\in\mathcal{X}, \theta_1+ \hdots + \theta_k=1, \theta_i\geq 0, i=1,\hdots, k}.
\end{equation*}
\end{definition}

\section{Conjugate duality in regularized IRL}\label{app:sec:convex_background}
In this section, we first recall some background from convex analysis and then briefly discuss the duality between reward equivalence classes and optimal occupancy measures.
\paragraph{Definitions}
We recall a few definitions related to convex functions. In convex analysis it is standard to consider extended real value functions $f:\R^n\to \Rbar\defeq [-\infty, \infty]$, where convex functions defined on some subset $\Xcal\subset \R^n$ are extended over the entire space by setting their value to $+\infty$ outside of their domain. The effective domain is defined as $\dom f \defeq \bc{x : f(x) < \infty}$, and a convex function $f$ is said to be proper if $f>-\infty$ and $\dom f \neq \emptyset$. Furthermore, $f$ is referred to as closed if its epigraph $\bc{\br{x, y}:x\in\dom f, y\geq f(x)}$ is a closed set.\footnote{A proper convex function is closed if and only if it is lower semi-continuous \citep{Rockafellar1970}.} In particular, $f$ is closed if it is continuous on $\dom f$ and $\dom f$ is a closed set \citep{boyd2004convex}. Lastly, we recall two key concepts in convex analysis -- the subdifferential and the convex conjugate of some convex function.

\begin{definition}[Subdifferential] 
A subgradient of $f:\R^n\to \Rbar$ at some point $x\in\R^n$ is a vector $g\in\R^n$ such that $f(x')\geq f(x) + g^\top\br{x- x}$ for all $x'\in\mathcal{X}$. The subdifferential $\partial f (x)$ at $x\in\mathcal{X}$ is the set of all subgradients at $x$, where $\partial f (x)$ is defined to be empty if $x\notin \dom f$. \looseness-1
\end{definition}
\begin{definition}[Convex Conjugate]
    The convex conjugate of $f:\R^n\to \Rbar$ is the function $f^*:\R^n\to \Rbar$ defined as
    \begin{equation}
    f^*(y) = \sup_{x} \ip{y}{x} - f(x).
    \end{equation}
\end{definition}

\paragraph{Key results} Next, we list two key results from convex analysis.
\begin{theorem}[\citep{Rockafellar1970}]\label{thm:differentiability}
    A function $f:\R^n\to\Rbar$ is differentiable at some point $x\in\dom f$ if and only if $\partial f(x)$ is singleton. In this case we have $\partial f(x) = \bc{\nabla f(x)}$.
\end{theorem}
\begin{theorem}[\citep{Rockafellar1970}]\label{thm:conjugate_duality}
    For any proper convex function $f:\R^n\to \Rbar$ it holds
    \begin{equation*}
        f^*(y) = \ip{y}{x} - f(x) \quad \iff \quad y \in \partial f(x).
    \end{equation*}
    If additionally $f$ is closed, then
    \begin{equation*}
        f^*(y) = \ip{y}{x} - f(x) \quad \iff \quad y \in \partial f(x) \quad \iff \quad x \in \partial f^*(y).
    \end{equation*}
\end{theorem}

\paragraph{Duality in IRL}\label{app:sec:consequences_irl}
Let $f:\R^{\SAcal}\to\Rbar$ be given by $f\defeq \occreg + \delta_{\Mcal}$, where $\delta_{\Mcal}$ is a characteristic function defined as $\delta_{\Mcal}(\mu)=0$ if $\mu\in\Mcal$ and $\delta_{\Mcal}(\mu)=\infty$, otherwise. Since $f$ is closed proper convex, Theorem~\ref{thm:differentiability} and \ref{thm:conjugate_duality} imply that for a strictly convex $\occreg$ we have
\begin{equation*}
    \RL(r) = \nabla f^*(r) \quad \text{ and } \quad \IRL(\mu) = \partial f(\mu) \cap \Rcal.
\end{equation*}
Additionally, under Assumption~\ref{ass:steep_regularization} and Slater's condition, which is ensured by Assumption~\ref{ass:occ_lower_bound}, we have $\partial f(\mu) = \nabla \occreg(\mu) + \Ucal$ \citep{pmlr-v202-schlaginhaufen23a}. Hence, $\mu$ is optimal for $r$ if and only if $r\in [\nabla \occreg(\mu)]_\Ucal $. Therefore, there is a one-to-one mapping between elements of the quotient space $\R^{\SAcal}/\Ucal$, i.e. reward equivalence classes, and corresponding optimal occupancy measures in $\Mcal$. This mapping is given by 
\begin{equation}
    \nabla f^*:\R^{\SAcal}/\Ucal\to\Mcal, [r]_\Ucal\mapsto\nabla f^*(r) = \RL(r),
\end{equation}
and its inverse by 
\begin{equation}
    \partial f:\Mcal\to \R^{\SAcal}/\Ucal, \mu\mapsto \partial f(\mu) = \nabla \occreg(\mu) + \Ucal.
\end{equation}

\section{Regularizers}\label{app:sec:regularizers}
We dedicate this section to discuss optimal policies in regularized MDPs and to recall their explicit form for Shannon and Tsallis entropy regularization.

\paragraph{Optimal policies} Throughout the appendix, it will convenient to use the notation $\pi_s\defeq \pi(\cdot|s)$. Given some proper closed strongly convex policy regularizer $\polreg$, it can be shown \citep{geist2019theory} that the optimal policy, $\pi^*$, satisfies
\begin{equation}\pi_s^*= \nabla \polreg^*(q^*_s) = \argmax_{\pi_s\in\Delta_{\Acal}} \ip{\pi_s}{q^*_s} - \polreg(\pi_s), \; \forall s\in \Scal,\end{equation}
where $q^*_s\in\R^{\Acal}$ is the optimal $q$-function defined via
\begin{equation}q^*_s(a) \defeq q^*(s,a) \defeq \max_{\pi\in\Delta_{\Acal}^{\Scal}} \E_{\pi}\bs{r(s_t, a_t) + \sum_{t\geq 1}\gamma^t\bs{r(s_t, a_t) - \polreg(\pi_s)}\bigg| s_0=s, a_0=a}.\end{equation}
Next, we discuss the explicit form of $\pi^*_s = \nabla \polreg^*(q^*_s)$ for the specific cases of Shannon and Tsallis-$1/2$ entropy regularization.

\paragraph{Shannon entropy} For some $\tau>0$, we define the Shannon entropy regularizer as $\polreg\defeq - \tau \Hcal$, where
\begin{equation}\Hcal(\pi_s)\defeq -\sum_a \pi_s(a)\log\pi_s(a),\end{equation}
is the Shannon entropy satisfying $0\leq \Hcal \leq \log|\Acal|$. It can be shown that $\polreg$ is $\tau$-strongly convex with respect to $\norm{\cdot}_1$ \citep{cover1999elements}, and the optimal policy satisfies \citep{geist2019theory}
\begin{equation}\pi^*(a|s) = \dfrac{\exp\br{q^*(s,a)/\tau}}{\sum_{a'}\exp\br{q^*(s,a')/\tau}}.\end{equation}

\paragraph{Tsallis entropy} For some parameter $\alpha\in\R$, the Tsallis entropy, $\Hcal_\alpha$, is defined as 
\begin{equation}\Hcal_\alpha(\pi_s) \defeq \dfrac{1}{\alpha-1}\br{1- \sum_a \pi_s(a)^\alpha}.\end{equation}
In the limit $\alpha\to 1$, the Tsallis entropy equals the Shannon entropy defined above. However, in this paper, we use \emph{Tsallis entropy} to refer to the choice $\alpha=1/2$, which is often adopted as regularization in multi-armed bandit and, more recently, policy optimization algorithms \citep{zimmert2021tsallis, dann2023best}. That is, we consider $\polreg(\pi_s) = -\tau \Hcal_{1/2}(\pi_s) = -2\tau \br{\sum_a \sqrt{\pi_s(a)}-1}$ for some $\tau>0$. We have $0 \leq -\polreg \leq 2\tau \br{\sqrt{|\Acal|}-1}$ and it can be shown that the optimal policy satisfies 
\begin{equation}\label{eq:tsallis_policy}
    \pi^*(a|s) = \br{\dfrac{\tau}{x_s - q^*(s,a)}}^2,
\end{equation}
where $x_s$ is a normalization parameter such that $\sum_a \pi^*(a|s) = 1$ \citep{zimmert2021tsallis}. Furthermore, since $\polreg$ has diagonal Hessian $\nabla^2\polreg(\pi_s)_{a,a}=\tau / (2\pi_s(a)^{3/2})$, it is $\tau/2$-strongly convex with respect to $\norm{\cdot}_2$ and hence also $\tau/(2|\Acal|)$-strongly convex with respect to $\norm{\cdot}_1$.

\section{Technical Lemmas}\label{app:sec:technical_lemmas}
In this section, we show several new technical results that are required for the proofs of our main theorems.
\subsection{Lipschitz continuity from policies to occupancy measures}
\begin{proposition}\label{prop:lipschitz-occ2policy}
    Let Assumption~\ref{ass:occ_lower_bound} hold.
    For any $\mu_1,\mu_2\in\Mcal$, we have
    \begin{equation} (1-\gamma)\norm{\mu_1 - \mu_2}_1 \leq \max_s\norm{\pi_s^{\mu_1}-\pi_s^{\mu_2}}_1 \leq \dfrac{2}{\numin} \norm{\mu_1 - \mu_2}_1.\end{equation}
\end{proposition}
\begin{proof}
To show the first inequality, we decompose
\begin{align*}
    \norm{\mu_1 - \mu_2}_1 &\leq \sum_{s,a} \abs{\nu_1(s)(\pi^{\mu_1}(a|s) - \pi^{\mu_2}(a|s))} + \sum_{s,a} \abs{(\nu_1(s) - \nu_2(s))\pi^{\mu_2}(a|s)}\\
    &\leq \max_s \norm{\pi^{\mu_1}_s - \pi^{\mu_2}_s}_1 + \norm{\nu_1 - \nu_2}_1\nonumber,
\end{align*}
where we used the triangle and Hölder's inequality. From the Bellman flow constraints,
\begin{equation*}
    \nu(s) = \gamma \sum_{s',a'}P(s|s',a')\mu(s',a') + (1-\gamma) \nu_0(s),
\end{equation*}
it follows that
\begin{align*}
    \norm{\nu_1 - \nu_2}_1 &= \gamma \sum_{s} \abs{\sum_{s',a'}P(s|s',a')(\mu_1(s',a')-\mu_2(s',a'))}\\
    &\leq \gamma \sum_{s',a'}\underbrace{\sum_s P(s|s',a')}_{=1} \abs{\mu_1(s',a')-\mu_2(s',a')}\nonumber\\
    &= \gamma\norm{\mu_1- \mu_2}_1\nonumber,
\end{align*}
where we again used the triangle inequality. Hence, we have
\begin{equation*}
    \max_s \norm{\pi^{\mu_1}_s - \pi^{\mu_2}_s}_1 \geq \norm{\mu_1 - \mu_2}_1 - \norm{\nu_1 - \nu_2}_1 \geq (1-\gamma) \norm{\mu_1- \mu_2}_1.
\end{equation*}
To show the second inequality, we use the reverse triangle inequality
\begin{align}
    \norm{\mu_1 - \mu_2}_1 &\geq \sum_{s,a} \abs{\nu_1(s)(\pi^{\mu_1}(a|s) - \pi^{\mu_2}(a|s))} - \sum_{s,a} \abs{(\nu_1(s) - \nu_2(s))\pi^{\mu_2}(a|s)}\\
    &= \sum_{s} \nu_1(s)\norm{\pi^{\mu_1}_s - \pi^{\mu_2}_s}_1 - \norm{\nu_1 - \nu_2}_1\nonumber,\\
    &\geq \numin \max_s\norm{\pi^{\mu_1}_s - \pi^{\mu_2}_s}_1 - \gamma\norm{\mu_1 - \mu_2}_1\nonumber,
\end{align}
where in the last step we used again that $\norm{\nu_1 - \nu_2}_1\leq \gamma \norm{\mu_1 - \mu_2}_1$.
By rearranging terms we arrive at the desired inequality
\begin{equation}
    \max_s\norm{\pi^{\mu_1}_s - \pi^{\mu_2}_s}_1 \leq \dfrac{1+\gamma}{\numin} \norm{\mu_1 - \mu_2}_1\leq \dfrac{2}{\numin} \norm{\mu_1 - \mu_2}_1.
\end{equation}
\end{proof}

\subsection{Strong convexity}\label{app:sec:strong_convexity}
Next, we show that strong convexity of the policy regularizer translates into strong convexity in the occupancy measure.
\begin{proposition}[Strong convexity] 
\label{prop:strong_convexity_general}
Let Assumption~\ref{ass:occ_lower_bound} hold and suppose that $\polreg$ is $\eta_{\polreg}$-strongly convex with respect to the $\norm{\cdot}_1$ norm. Then, $\occreg$ is $\eta_{\polreg}\numin/H_\gamma^2$-strongly convex with respect to $\norm{\cdot}_1$.
\end{proposition}
\begin{proof}
    We need to show that for $\alpha\in(0,1), \alphabar=1-\alpha$ and any two $\mu_1,\mu_2\in\mathcal{M}$ it holds that
    \begin{equation}
    \occreg(\alpha \mu_1 + \alphabar \mu_2) \leq \alpha\occreg(\mu_1) + \alphabar\occreg(\mu_2) - \dfrac{\alpha \alphabar \eta}{2} \norm{\mu_1-\mu_2}_1^2,
\end{equation}
for $\eta = \eta_{\polreg}\numin/H_\gamma^2$. To this end, we start similarly as in the proof of strict convexity by \citet{pmlr-v202-schlaginhaufen23a}, but use $\nu(s)\geq \numin$ and strong convexity of $\polreg$.
\begin{align}\label{eq:strong_convexity_eq1}
    &\occreg(\alpha\mu_1+ \bar{\alpha}\mu_2) \\
    =& \sum_{s} \br{\alpha\nu_1(s) + \alphabar\nu_2(s)}\polreg\br{\dfrac{\alpha\mu_1(s,\cdot) + \alphabar\mu_2(s,\cdot)}{\alpha\nu_1(s) + \alphabar\nu_2(s)}}\nonumber\\
    =& \sum_{s} \br{\alpha\nu_1(s) + \alphabar\nu_2(s)}\polreg\br{\dfrac{\alpha\mu_1(s, \cdot)}{\alpha\nu_1(s) + \alphabar\nu_2(s)}\dfrac{\nu_1(s)}{\nu_1(s)} + \dfrac{\alphabar\mu_2(s, \cdot)}{\alpha\nu_1(s) + \alphabar\nu_2(s)}\dfrac{\nu_2(s)}{\nu_2(s)}}\nonumber\\
    =& \sum_{s} \br{\alpha\nu_1(s) + \alphabar\nu_2(s)}\polreg\br{\underbrace{\dfrac{\alpha\nu_1(s)}{\alpha\nu_1(s) + \alphabar\nu_2(s)}}_{\beta_s}\pi^{\mu_1}_s + \underbrace{\dfrac{\alphabar\nu_2(s)}{\alpha\nu_1(s) + \alphabar\nu_2(s)}}_{1-\beta_s}\pi^{\mu_2}_s}\nonumber\\
    \leq& \sum_{s} \br{\alpha\nu_1(s) + \alphabar\nu_2(s)}\br{\beta_s\polreg\br{\pi^{\mu_1}_s}+(1-\beta_s)\polreg\br{\pi^{\mu_2}_s} - \dfrac{\beta_s (1-\beta_s) \eta_{\polreg}}{2}\norm{\pi^{\mu_1}_s - \pi^{\mu_2}_s}^2_1}\nonumber\\
    =& \sum_{s} \br{\alpha\nu_1(s)\polreg\br{\pi^{\mu_1}_s}+\alphabar\nu_2(s)\polreg\br{\pi^{\mu_2}_s} - \dfrac{\alpha \alphabar\eta_{\polreg} }{2}\dfrac{\nu_1(s) \nu_2(s)}{\alpha\nu_1(s) + \alphabar\nu_2(s)}\norm{\pi^{\mu_1}_s - \pi^{\mu_2}_s}^2_1}.
\end{align}
From here on, we use that
\begin{align}\label{eq:strong_convexity_eq2}
    &\sum_s \dfrac{\nu_1(s) \nu_2(s)}{\alpha\nu_1(s) + \alphabar\nu_2(s)}\norm{\pi^{\mu_1}_s - \pi^{\mu_2}_s}^2_1 \nonumber\\
    &=  \sum_s \underbrace{\dfrac{\max\bc{\nu_1(s),\nu_2(s)}}{\alpha\nu_1(s) + \alphabar\nu_2(s)}}_{\geq 1}\underbrace{\min\bc{\nu_1(s),\nu_2(s)}}_{\geq \numin}\norm{\pi^{\mu_1}_s - \pi^{\mu_2}_s}^2_1\nonumber\\
    &\geq \sum_s \numin \norm{\pi^{\mu_1}_s - \pi^{\mu_2}_s}^2_1\nonumber\\
    &\geq \numin \max_s \norm{\pi^{\mu_1}_s - \pi^{\mu_2}_s}^2_1\nonumber\\
    &\geq \numin /H_\gamma^2 \norm{\mu_1-\mu_2}_1^2.
\end{align}
where we used Proposition~\ref{prop:lipschitz-occ2policy} in the last step. Plugging the inequality \eqref{eq:strong_convexity_eq2} back into \eqref{eq:strong_convexity_eq1} concludes the proof.
\end{proof}

\subsection{Lipschitz gradients}\label{app:sec:lipschitz}
In this section, we show how we can get bounds on the Lipschitz constant $L_{\Kcal}$. To this end, we first need to lower bound the optimal policies.
\paragraph{Policy lower bounds}
The following proposition establishes a lower bound for optimal policies with Shannon and Tsallis entropy regularization.
\begin{proposition}\label{prop:policy_lower_bound}
    Let $H_\gamma = 1/(1-\gamma)$ and $r_{\max} \defeq \norm{r}_\infty$. Then, we have the following lower bounds:
    \begin{enumerate}[a)]
        \item If $\polreg = - \tau \Hcal$, then $\pi^*(a|s) \geq \exp\br{-2r_{\max} H_\gamma/\tau}/|\Acal|^{1+H_\gamma}$.
        \item If $\polreg = - \tau \Hcal_{1/2}$, then $\pi^*(a|s) \geq \br{\br{2r_{\max} /\tau + 3\sqrt{|\Acal|}} H_\gamma}^{-2}$.
    \end{enumerate}
\end{proposition}
\begin{proof}  Recall the formula for the optimal policies in Appendix~\ref{app:sec:regularizers}.
    
    \emph{Part a):} Since, $-r_{\max} H_\gamma \leq q^*(s,a) \leq \br{r_{\max} + \tau \log|\Acal|}H_\gamma$, it holds that
    \begin{align}
        \pi^*(a|s) &\geq \dfrac{\exp\br{-r_{\max} H_\gamma/\tau}}{|\Acal|\exp\br{\br{r_{\max} + \tau \log|\Acal|}H_\gamma/\tau}} \\
        &= \dfrac{\exp\br{-r_{\max} H_\gamma/\tau}}{|\Acal|^{1+H_\gamma}\exp\br{r_{\max} H_\gamma/\tau}} = \dfrac{\exp\br{-2r_{\max} H_\gamma/\tau}}{|\Acal|^{1+H_\gamma}}.
    \end{align}
    \emph{Part b):} The proof of b) is similar to \citep[Lemma~8]{ouhamma2023learning}. However, our settings are slightly different. Recall that
    \begin{equation*}
        \pi^*(a|s) = \br{\dfrac{\tau}{x_s - q^*(s,a)}}^2.
    \end{equation*}  
    By \citet[Lemma 10]{ouhamma2023learning} we have $\tau \leq x_s - \norm{q_s^*}_\infty \leq \tau \sqrt{|\Acal|}$. Furthermore, it holds that $-r_{\max} H_\gamma \leq q^*(s,a) \leq \br{r_{\max} + 2\tau \sqrt{|\Acal|}} H_\gamma$. Hence, we have 
    \begin{equation}0<x_s - q^*(s,a) \leq \tau \sqrt{|\Acal|} +  \br{r_{\max} + 2\tau \sqrt{|\Acal|}} H_\gamma + r_{\max} H_\gamma \leq \br{2r_{\max} + 3\tau \sqrt{|\Acal|}} H_\gamma,\end{equation}
    which yields the desired lower bound.    
\end{proof}
We also highlight the following result, which shows that if the policy $\pi^\mu$ is lower bounded on some set $\Kcal\subset \Mcal$, then it is also lower bounded on its convex hull $\conv\Kcal$.
\begin{proposition}
\label{prop:lines2lines}
    Suppose $\mu = \alpha \mu_1 + (1-\alpha) \mu_2$ with $\alpha \in (0,1)$ and $\mu_1, \mu_2\in\Mcal$. Then, 
    \begin{equation}\pi_s^{\mu} = \underbrace{\dfrac{\alpha\nu_1(s)}{\alpha\nu_1(s) + \alphabar\nu_2(s)}}_{\beta_s}\pi^{\mu_1}_s + \underbrace{\dfrac{\alphabar\nu_2(s)}{\alpha\nu_1(s) + \alphabar\nu_2(s)}}_{1-\beta_s}\pi^{\mu_2}_s,\end{equation}
    where $\beta_s \in (0,1)$.
\end{proposition}
\vspace{-0.2cm}
\begin{proof}
    The proof follows immediately from the proof of Proposition~\ref{prop:strong_convexity_general}.
\end{proof}

\paragraph{Hessian upper bounds}
In the following, we establish upper bounds for the Hessians of the occupancy measure regularizations, $\occreg$, resulting from Shannon and Tsallis entropy regularization of the policy. In particular, we aim to upper-bound the maximum norm of the Hessian in terms of the smallest entry of the policy.
\begin{proposition}\label{prop:hessian_upper_bound}
    Let Assumption~\ref{ass:occ_lower_bound} hold. Consider $\mu\in\Mcal$ and let $\pimin = \min_{s,a}\pi^\mu(a|s)>0$. Then, the Hessian of $\occreg$ is upper bounded as follows:
    \begin{enumerate}[a)]
        \item If $\polreg = - \tau \Hcal$, then $\norm{\nabla^2 \occreg(\mu)}_\infty\leq \frac{\tau}{\numin \pimin}$.
        \item If $\polreg = - \tau \Hcal_{1/2}$, then $\norm{\nabla^2 \occreg(\mu)}_\infty\leq \frac{\tau}{\numin \pimin^{3/2}}$. 
    \end{enumerate}
    Here, $\norm{\cdot}_\infty$ denotes the maximum norm $\norm{A}_\infty=\max_{ij} \abs{A_{ij}}$.
\end{proposition}
\vspace{-0.2cm}
\begin{proof}
    As shown by \citet[Proposition B.2]{pmlr-v202-schlaginhaufen23a}, the gradient of $\occreg$ satisfies
    \begin{equation}\label{eq:occreg_gradient}
        \nabla\occreg(\mu)(s,a) = \polreg(\pi_s^\mu) + \nabla\polreg(\pi_s^\mu)(a) - \ip{\nabla\polreg(\pi_s^\mu)}{\pi_s^\mu}.
    \end{equation}
    Moreover, we have
    \begin{equation}
        \dfrac{\partial \pi^{\mu}(s,a)}{\partial \mu(s',a')} = \delta_{s,s'}\cdot\dfrac{\delta_{a,a'}-\pi^{\mu}(a|s)}{\nu(s)}.
    \end{equation}
    Using the above two formulas, we can calculate the Hessians explicitly.
    
    \emph{Part a):}
    For the Shannon entropy it holds by \eqref{eq:occreg_gradient} that $\nabla\occreg(\mu)(s,a)=\tau \log \pi^\mu(a|s)$. Hence, 
    \begin{equation}\dfrac{\partial^2 \occreg(\mu)}{\partial \mu(s',a')\partial \mu(s,a)} = \tau\cdot\dfrac{1}{\pi^{\mu}(a|s)}\cdot\delta_{s,s'}\cdot\dfrac{\delta_{a,a'}-\pi^{\mu}(a|s)}{\nu(s)},\end{equation}
    and
    \begin{equation*}
        \abs{\nabla^2\occreg(\mu)_{(s',a'),(s,a)}} = \abs{\dfrac{\partial^2 \occreg(\mu)}{\partial \mu(s',a')\partial \mu(s,a)}} \leq \dfrac{\tau}{\numin\pimin}.
    \end{equation*}
    \emph{Part b):} For the Tsallis entropy it holds by \eqref{eq:occreg_gradient} that 
    \begin{equation*}
        \nabla \occreg(\mu)(s,a) = - \tau\br{\sum_{a''}\sqrt{\pi^\mu(a''|s)} + \dfrac{1}{\sqrt{\pi^\mu(a|s)}} - 2}.
    \end{equation*}
    Therefore, the second derivative is bounded as follows
    \begin{align*}
        \abs{\dfrac{\partial^2 \occreg(\mu)}{\partial \mu(s',a')\partial \mu(s,a)}} &= \Bigg| - \tau\cdot\delta_{s,s'}\cdot \Bigg(\sum_{a''}\dfrac{1}{2\sqrt{\pi^\mu(a''|s)}}\dfrac{\delta_{a',a''} -\pi^{\mu}(a''|s)}{\nu(s)} \\ 
        &\hspace{5.3cm}- \dfrac{1}{2 \pi^\mu(a|s)^{3/2}}\dfrac{\delta_{a,a'}-\pi^{\mu}(a|s)}{\nu(s)}\Bigg) \Bigg|\\
        &= \dfrac{\tau\cdot\delta_{s,s'}}{2\nu(s)}\abs{\dfrac{1}{\sqrt{\pi^\mu(a'|s)}} - \sum_{a''} \sqrt{\pi^\mu(a''|s)} - \dfrac{\delta_{a,a'}}{\pi^\mu(a|s)^{3/2}} + \dfrac{1}{\sqrt{\pi^\mu(a|s)}}}\\
        &\leq \dfrac{\tau}{2\numin}\br{\underbrace{\abs{\dfrac{1}{\sqrt{\pi^\mu(a'|s)}} - \sum_{a''} \sqrt{\pi^\mu(a''|s)}}}_{(A)} + \underbrace{\abs{\dfrac{\delta_{a,a'}}{\pi^\mu(a|s)^{3/2}} - \dfrac{1}{\sqrt{\pi^\mu(a|s)}}}}_{(B)}}.
    \end{align*}
    Now, since $\sum_a \sqrt{\pi^\mu(a|s)} \leq \sqrt{A} \leq 1/\sqrt{\pimin}$, we have $(A)+(B)\leq 1/\sqrt{\pimin} + 1/\pimin^{3/2}\leq 2/\pimin^{3/2}$, which yields the desired result 
    \begin{equation}
        \abs{\nabla^2\occreg(\mu)_{(s',a'),(s,a)}}\leq \tau/\br{\numin \pimin^{3/2}}.
    \end{equation}
\end{proof}

\vspace{-0.2cm}
\paragraph{Lipschitz gradients}
Next, we provide explicit Lipschitz constants $L_{\Kcal}$ for $\nabla\occreg$ corresponding to Shannon and Tsallis entropy regularization.
\begin{restatable}[Lipschitz gradients]{proposition}{lipschitzgradients}
\label{prop:lipschitz_gradients}%
Consider some closed convex set $\Kcal\subset \Mcal$ and suppose $\pimin = \min_{\mu\in\Kcal}\min_{s,a} \pi^\mu(a|s)>0$. Then, the gradient of $\occreg$ is Lipschitz continuous over $\Kcal$ i.e. 
\begin{equation}\norm{\nabla\occreg(\mu_1)-\nabla\occreg(\mu_2)}_2 \leq L_{\Kcal} \norm{\mu_1 -\mu_2}_2, \quad \forall \mu_1, \mu_2 \in\Kcal,\end{equation}
where the respective Lipschitz constants are as follows:
\begin{enumerate}[a)]
    \item If $\polreg = - \tau \Hcal$, then $L_{\Kcal} = \tau|\Scal||\Acal| / (\numin\pimin)$.
    \item If $\polreg = - \tau \Hcal_{1/2}$, then $L_{\Kcal} = \tau|\Scal||\Acal| / (\numin\pimin^{3/2})$.
\end{enumerate}
\end{restatable}
\begin{proof}
    Defining $h=\mu_2 - \mu_1$, Lipschitz continuity follows from
    \begin{align}
        \norm{\nabla\occreg(\mu_1)-\nabla\occreg(\mu_2)}_2 &\stackrel{(i)}{\leq} \sqrt{|\Scal||\Acal|}\norm{\nabla\occreg(\mu_1)-\nabla\occreg(\mu_2)}_\infty \\
        &\stackrel{(ii)}{=} \sqrt{|\Scal||\Acal|}\norm{\int_0^1\nabla^2\occreg(\mu_1 + th)h dt}_\infty \\
         &\stackrel{(iii)}{\leq} \sqrt{|\Scal||\Acal|}\int_0^1\norm{\nabla^2\occreg(\mu_1 + th)h}_\infty dt \\
         &\stackrel{(iv)}{\leq} \sqrt{|\Scal||\Acal|}\int_0^1\norm{\nabla^2\occreg(\mu_1 + th)}_\infty \norm{h}_1 dt\\
         &\stackrel{(v)}{\leq} |\Scal||\Acal| \max_{0\leq t \leq 1}\norm{\nabla^2\occreg(\mu_1 + th)}_\infty \norm{\mu_1 - \mu_2}_2.
    \end{align}
    Here, we used in $(i)$ and $(v)$ that $\norm{x}_1 \leq \sqrt{n}\norm{x}_2 \leq n\norm{x}_\infty$ for $x\in\R^n$, and in $(ii)$ we applied the fundamental theorem of calculus. Moreover, $(iii)$ follows from $\abs{\int f}\leq \int \abs{f}$, and $(iv)$ from Hölder's inequality. Now, by convexity $\mu_1, \mu_2 \in\Kcal$ implies that $\mu_1 + th\in\Kcal$ for $t\in[0,1]$. Hence, plugging in the upper bounds from Proposition~\ref{prop:hessian_upper_bound} concludes the proof.
\end{proof}

\subsection{Dual smoothness and strong convexity}\label{app:sec:dual_smoothness_strong_convexity}
Next, we show that the convex conjugate, $f^*$, of the extended real value function $f\defeq \occreg + \delta_{\Mcal}$ (see Appendix~\ref{app:sec:convex_background}) is -- if understood as a mapping from $\R^{\SAcal} / \Ucal$ to $\R$ -- both smooth and strongly convex on $\Rcal$ with respect to the quotient norm. While it is well-known that global smoothness and strong convexity are dual properties \citep[Proposition 12.60]{rockafellar2009variational}, the key challenge for proving dual strong convexity is that $\occreg$ has only locally Lipschitz gradients (see Proposition~\ref{prop:lipschitz_gradients}). Proposition~\ref{prop:dual_strong_convexity_smoothness} below shows that $\eta$-strong convexity of $\occreg$ implies dual $1/\eta$-smoothness and locally Lipschitz gradients imply dual $\sigma_{\Rcal}$-strong convexity on $\Rcal$ for some $\sigma_{\Rcal}>0$. Moreover, we provide explicit lower bounds on $\sigma_{\Rcal}$ for Shannon and Tsallis entropy entropy regularization.
\begin{restatable}{proposition}{dualstrongconvexitysmoothness}
    \label{prop:dual_strong_convexity_smoothness}%
    Let $f^*$ be the convex conjugate of $f\defeq \occreg + \delta_{\Mcal}$. Then, the following holds:
    \begin{enumerate}[a)]
        \item Suppose that $\occreg$ is $\eta$-strongly convex over $\Mcal$, that is for all $\mu,\mu'\in\Mcal$ it holds that
        \begin{equation*}
            \occreg(\mu') \geq \occreg(\mu) + \ip{\nabla \occreg(\mu)}{\mu'-\mu} + \dfrac{\eta}{2} \norm{\mu-\mu'}_2^2,
        \end{equation*}
        then we have for all $r,r'\in\R^{\SAcal}$ that
        \begin{equation}\label{eq:dual_smoothness}
            f^*(r') \leq  f(r) + \ip{\nabla f^*(r)}{r'-r} + \dfrac{1}{2\eta} \norm{[r']_{\Ucal}-[r]_{\Ucal}}_{2}^2.
        \end{equation}
        \item Suppose that for any closed convex subset $\mathcal{K}\subset \relint \Mcal$, there is some $L_{\mathcal{K}}>0$ such that for all $\mu, \mu'\in\mathcal{K}$ it holds that
        \begin{equation*}
            \norm{\nabla\occreg(\mu) - \nabla\occreg(\mu')}_2 \leq L_{\Kcal} \norm{\mu-\mu'}_2,
        \end{equation*}
        then we have for all $r,r'\in\Rcal$ that
        \begin{equation}\label{eq:dual_sc}
            f^*(r') \geq f^*(r) + \ip{\nabla f^*(r)}{r'-r} + \dfrac{\sigma_{\Rcal}}{2}\norm{[r']_{\Ucal}-[r]_{\Ucal}}_{2}^2,
        \end{equation}
        for some $\sigma_{\Rcal}>0$.
        \item Let $H_\gamma \defeq 1/(1-\gamma)$, $\rmax\defeq\max_{r\in\Rcal}\norm{r}_\infty$, $\DR = \max_{r,r'\in\Rcal} \norm{r-r'}_2$, and suppose that $\tau<D$, then for the Shannon entropy the inequality \eqref{eq:dual_sc} holds with
        \begin{equation}
            \sigma_{\Rcal} = \dfrac{\numin\exp\br{\frac{-2R H_\gamma}{\tau}}}{2D |\Scal||\Acal|^{2+H_\gamma}},
        \end{equation}
        and for the Tsallis entropy with
        \begin{equation}
             \sigma_{\Rcal} = \dfrac{\numin}{2\sqrt{2}D |\Scal||\Acal|\br{\br{2R /\tau + 3\sqrt{|\Acal|}} H_\gamma}^3 }.
        \end{equation}
    \end{enumerate}
\end{restatable}
\begin{proof} 

    \textit{Part a):}
        Since $f$ is $\eta$-strongly convex with respect to $\norm{\cdot}_2$, the convex conjugate $f^*$ is $1/\eta$-smooth with respect to the dual norm in $\R^{\SAcal}/\Ucal$  \citep[Proposition 12.60]{rockafellar2009variational}, which is equivalent to \eqref{eq:dual_smoothness}.
        
    \textit{Part b):}
    The show b), we closely follow \citep[Theorem 4.1]{goebel2008local}, but we need to account for the quotient spaces. We define the sets $\Kcal = \nabla f^*(\Rcal) =\RL(\Rcal)$ and $\Kcal_{\epsilon} = \conv( \Kcal )+ \epsilon (\Bcal \cap \operatorname{aff}(\Mcal))$, where $\Bcal\subset\R^{\SAcal}$ denotes the closed unit ball with respect to $\norm{\cdot}_2$ and $\epsilon>0$ is chosen such that $\Kcal_{\epsilon}\subset \relint\Mcal$. Moreover, we let $L$ be the Lipschitz constant of $\nabla \occreg$ over $\Kcal_{\epsilon}$. Now, consider $r\in\Rcal$ and $\mu=\nabla f^*(r)$. Then, for any $r'\in\Rcal$, we have
    \begin{align*}\label{eq:local_sc_proof}
        f^*(r') &= \sup_{\mubar}\bs{\ip{r'}{\mubar}-f(\mubar)}\\
        &\stackrel{(i)}{\geq} \sup_{\mubar\in\Kcal_{\epsilon}}\bs{\ip{r'}{\mubar}-f(\mubar)}\\
        &\stackrel{(ii)}{\geq} \sup_{\mubar\in\Kcal_{\epsilon}}\bs{\ip{r'}{\mubar}-f(\mu) - \ip{r}{\mubar-\mu}-\dfrac{L}{2}\norm{\mubar-\mu}_2^2}\\
        &\stackrel{(iii)}{=} \ip{r}{\mu} - f(\mu) + \sup_{\mubar\in\Kcal_{\epsilon}}\bs{\ip{r'-r}{\mubar}-\dfrac{L}{2}\norm{\mubar-\mu}_2^2}\\
        &\stackrel{(iv)}{=} f^*(r) + \sup_{\mubar\in\Kcal_{\epsilon}}\bs{\ip{r'-r}{\mubar}-\dfrac{L}{2}\norm{\mubar-\mu}_2^2}\\
        &\stackrel{(v)}{=} f^*(r) + \ip{r'-r}{\mu} + \sup_{\mubar\in\Kcal_{\epsilon}}\bs{\ip{r'-r}{\mubar-\mu}-\dfrac{L}{2}\norm{\mubar-\mu}_2^2}\numberthis.
    \end{align*}
\noindent Here, $(i)$ follows from $\Kcal_{\epsilon}\subset \R^{\SAcal}$, $(ii)$ from the fact that Lipschitz gradients imply that
\begin{equation}f(\mubar)\leq f(\mu) + \ip{g}{\mubar-\mu} + \dfrac{L}{2} \norm{\mubar-\mu}_2^2,\end{equation} 
for any $g\in\partial f(\mu)$ \citep[Lemma 5.7]{beck2017first} and $r\in\partial f(\mu)$ (see Theorem~\ref{thm:conjugate_duality}). Moreover, $(iii)$ and $(v)$ follow from rearranging terms, and $(iv)$ from $f^*(r)=\ip{r}{\mu} - f(\mu)$. Now, consider any $\alpha>0$ such that $\sigma_{\Rcal}=2(\alpha - L\alpha^2/2)>0$ and $\alpha D\leq  \epsilon$. Setting $\mubar\in\Kcal_{\epsilon}$ in the above supremum to $(\mubar-\mu)= \alpha \Pi_{\Ucal^{\perp}}(r'-r)$ yields the desired result
\begin{align}
    f^*(r') &\geq f^*(r) + \ip{r'-r}{\nabla f^*(r)} + \alpha\ip{r'-r}{\Pi_{\Ucal^{\perp}}(r'-r)} - \dfrac{L\alpha^2}{2}\norm{\Pi_{\Ucal^{\perp}}(r'-r)}_2^2\\
    &= f^*(r) + \ip{r'-r}{\nabla f^*(r)} + \dfrac{\sigma_{\Rcal}}{2} \norm{[r']_{\Ucal} - [r]_{\Ucal}}_{2}^2.
\end{align}
Note that we indeed have $\mubar\in\Kcal_{\epsilon}$ as $\mu\in\Kcal$ and $\norm{\mu-\mubar}_2\leq \alpha \norm{r-r'}_2\leq \alpha D\leq \epsilon$.

\emph{Part c):} To get an explicit constant for $\sigma_{\Rcal}$, we need to appropriately choose $\epsilon$ and calculate the corresponding Lipschitz constant. To this end, we first recall that according to Proposition~\ref{prop:policy_lower_bound} and \ref{prop:lines2lines} policies corresponding to occupancy measures in $\conv\Kcal$ are, for Shannon and Tsallis entropy, lower bounded by
\begin{equation}
    \piminsh = \dfrac{\exp\br{-2R H_\gamma/\tau}}{|\Acal|^{1+H_\gamma}} \quad \text{and} \quad \pimints = \br{\br{2R /\tau + 3\sqrt{|\Acal|}} H_\gamma}^{-2},
\end{equation}
respectively. Furthermore, for any $\mu \in \Kcal_\epsilon$, we have by Proposition~\ref{prop:lipschitz-occ2policy} and equivalence of norms
\begin{equation}
    \pi^\mu(a|s)\geq \pimin - \dfrac{2\sqrt{|\Scal||\Acal|}}{\numin} \epsilon = \pimin'.
\end{equation}
Hence, by setting $\epsilon=\numin\pimin/(4\sqrt{|\Scal||\Acal|})$, we have $\pi^\mu(a|s)\geq \pimin' = \pimin/2$ for any $\mu\in\Kcal_\epsilon$. As for the Lipschitz constant over $\Kcal_\epsilon$, we have by Proposition~\ref{prop:lipschitz_gradients}
\begin{equation}\label{eq:ref1}
    L_{\text{Sh}} = \dfrac{\tau |\Scal||\Acal|}{\numin \piminsh'} \quad \text{and} \quad L_{\text{Ts}} = \dfrac{\tau |\Scal||\Acal|}{\numin {\pimints'}^{3/2}},
\end{equation}
for the Shannon and Tsallis entropy, respectively. Now, we need to ensure that $\alpha>0$ such that $\sigma_{\Rcal}=2(\alpha - L\alpha^2/2)>0$ and $\alpha \leq  \epsilon /D$. To that end, we set for both regularizations $\alpha = \tau / (LD)$, which in light of \eqref{eq:ref1} ensures that
\begin{equation}\label{eq:dual_sc_constant_proof}
    \alpha = \dfrac{\tau}{LD}  \leq \dfrac{\numin \pimin'}{D |\Scal||\Acal|} \leq \dfrac{\numin \pimin'}{2D\sqrt{|\Scal||\Acal|}} = \dfrac{\epsilon}{D},
\end{equation}
for $|\Scal|,|\Acal|\geq 2$. Moreover, we get the dual strong convexity constant
\begin{equation}
    \sigma_{\Rcal} = 2\br{\dfrac{\tau}{LD} - \dfrac{\tau^2}{2LD^2}} = \dfrac{2\tau}{LD}\br{1-\dfrac{\tau}{2D}},
\end{equation}
which is larger than $\tau/LD$ for $\tau < D$. We can therefore choose $\sigma_{\Rcal} = \tau/LD$ as a dual strong convexity constant. Plugging in the Lipschitz constants for the two regularizations yields
\begin{equation}
    \sigma_{\Rcal, \text{Sh}} = \dfrac{\piminsh'\numin}{D |\Scal||\Acal|} = \dfrac{\piminsh\numin}{2D |\Scal||\Acal|} = \dfrac{\exp\br{-2R H_\gamma/\tau}\numin}{2D |\Scal||\Acal|^{2+H_\gamma}},
\end{equation}
for the Shannon entropy, and 
\begin{align}
    \sigma_{\Rcal, \text{Ts}} &= \dfrac{{\pimints'}^{3/2}\numin}{D |\Scal||\Acal|} = \dfrac{\pimints^{3/2}\numin}{ 2\sqrt{2} D |\Scal||\Acal|} = \dfrac{\numin}{2\sqrt{2}D |\Scal||\Acal|\br{\br{2R /\tau + 3\sqrt{|\Acal|}} H_\gamma}^3 },
\end{align}
for the Tsallis entropy.
\end{proof}

\begin{remark}[Large $\tau$ regime]
    Note that if $\tau \geq 2D/\sqrt{|\Scal||\Acal|}$, we can set $ \sigma_{\Rcal}=\alpha = 1/L$, while still satisfying the condition \eqref{eq:dual_sc_constant_proof}. This leads to the strong convexity constants
    \begin{equation*}
        \sigma_{\Rcal, \text{Sh}} = \dfrac{\exp\br{-2R H_\gamma/\tau}\numin}{2\tau |\Scal||\Acal|^{2+H_\gamma}},\;  \sigma_{\Rcal, \text{Ts}} = \dfrac{\numin}{2\sqrt{2}\tau |\Scal||\Acal|\br{\br{2R /\tau + 3\sqrt{|\Acal|}} H_\gamma}^3 }.
    \end{equation*}
\end{remark}

\subsection{Regularity constants}\label{app:sec:regularity}
In the following Proposition, we summarize the regularity constants for Shannon and Tsallis entropy regularization. We highlight that these constants are lower bounds for $\eta$ and $\sigma_{\Rcal}$.
\begin{restatable}{proposition}{strongconvexitylipschitzgradients}
\label{prop:strong_convexity_lipschitz_gradients}%
Let $H_\gamma \defeq 1/(1-\gamma)$, $\rmax\defeq\max_{r\in\Rcal}\norm{r}_\infty$, and $\DR = \max_{r,r'\in\Rcal} \norm{r-r'}_2$. Suppose that $\tau < D$, then for the Shannon entropy, Assumption~\ref{ass:regularity} holds with 
\begin{equation}
    \eta = \tau \numin / H_\gamma^2  \quad \text{and} \quad  \sigma_{\Rcal} = \dfrac{\exp\br{-2R H_\gamma/\tau}\numin}{2D |\Scal||\Acal|^{2+H_\gamma}},
\end{equation}
and for the Tsallis entropy with
\begin{equation}
    \eta= \tau\numin / (2H_\gamma^2 |\Acal|) \quad \text{and} \quad 
    \sigma_{\Rcal} =  \dfrac{\numin}{2\sqrt{2}D |\Scal||\Acal|\br{\br{2R /\tau + 3\sqrt{|\Acal|}} H_\gamma}^3 }.
\end{equation}
\end{restatable}
\begin{proof}
    The derivation for $\eta$ is given in Proposition~\ref{prop:strong_convexity_general} and for $\sigma_{\Rcal}$ in Proposition~\ref{prop:dual_strong_convexity_smoothness} above.
\end{proof}
\subsection{Suboptimality bounds}\label{app:sec:suboptimality_bounds}
\bregmandivergence*
\begin{proof}
    Let $\mu = \RL(r)$. We have
    \begin{align}
    \mathsf{SubOpt}(r, \mu') &= \ip{r}{\mu - \mu'} - \occreg(\mu) + \occreg(\mu')\\
    &= \ip{\nabla \occreg(\mu)}{\mu - \mu'} - \occreg(\mu) + \occreg(\mu')\\
    &= D_{\occreg}(\mu', \mu),
    \end{align}
    where the second equality holds, as by \eqref{eq:feasible_reward_set_explicit} we have $r-\nabla \occreg(\mu)\in \Ucal$, and $\mu - \mu' \in \Ucal^\perp$.
\end{proof}

\bregmanbound*
\begin{proof}
    Let $f\defeq \occreg + \delta_{\Mcal}$ and $\mu = \RL(r), \mu'=\RL(r')$. We then have
    \begin{align*}
        D_{\occreg}(\mu, \mu') &\stackrel{(i)}{=} f(\mu) - f(\mu') - \ip{r'}{\mu- \mu'}\\
        &\stackrel{(ii)}{=} f^*(r') - \ip{r'}{\mu'} - f^*(r) + \ip{r}{\mu} - \ip{r'}{\mu- \mu'}\\
        &\stackrel{(iii)}{=} f^*(r') - f^*(r) - \ip{r'-r}{\nabla f^*(r)} = D_{f^*}(r',r).
    \end{align*}
    Here, $(i)$ follows from the definition of $f$ and $r'\in [\nabla \occreg(\mu')]_{\Ucal}$, in $(ii)$ we use that $f(\mu) = \ip{r}{\mu} - f^*(r)$ and $f(\mu') = \ip{r'}{\mu'} - f^*(r')$, and $(iii)$ follows from rearranging terms and $\mu=\nabla f^*(r)$. The result then follows from dual strong convexity and smoothness as established in Proposition~\ref{prop:dual_strong_convexity_smoothness}.
\end{proof}
Note that without steep regularization it is impossible to lower bound the suboptimality in terms of reward distances in $\R^{\SAcal}/\Ucal$ (Proposition~\ref{prop:bregman_divergence} doesn't hold). However, we still have the following upper bound.
\begin{proposition}\label{prop:unreg_subopt_bound}
Consider an arbitrary regularization and let $\mu\in\RL(r), \mu'\in\RL(r')$. Then,
\begin{equation}
    \mathsf{SubOpt}(r, \mu') \leq 2 \norm{[r]_{\Ucal} - [r']_{\Ucal}}_\infty \leq 2 \norm{[r]_{\Ucal} - [r']_{\Ucal}}_2.
\end{equation}
\end{proposition}
\begin{proof}
Let $r''\defeq\argmin_{\tilde r\in [r']_{\Ucal}} \norm{\tilde r - r}_\infty$, then the following holds
\begin{align*}
        \mathsf{SubOpt}(r, \mu') &= \max_{\mu\in\Mcal} J(r, \mu) - J(r, \mu') \\
        &\stackrel{(i)}{\leq} \abs{\max_{\mu\in\Mcal} J(r, \mu) - \max_{\mu\in\Mcal} J(r'', \mu)} + \abs{J(r'', \mu') - J(r, \mu')}\\
        &\stackrel{(ii)}{\leq} \max_{\mu\in\Mcal} \abs{\ip{r - r''}{\mu}} + \abs{\ip{r - r''}{\mu'}}\\
        &\stackrel{(iii)}{\leq} 2\norm{r - r''}_\infty \stackrel{(iv)}{=} 2\norm{[r]_{\Ucal} - [r']_{\Ucal}}_\infty \leq 2\norm{[r]_{\Ucal} - [r']_{\Ucal}}_2.\numberthis
\end{align*}
Here, $(i)$ follows from the triangle inequality  and optimality of $\mu'$, $(ii)$ from $\abs{\max f- \max g}\leq \max\abs{f-g}$ and simplifying, $(iii)$ from Hölder's inequality, and $(iv)$ from the definition of $r''$ and the quotient norm.
\end{proof}

\subsection{Perturbation bounds}\label{app:sec:perturbation_bounds}
Next, we provide a bound quantifying the change in the quotient norm when changing the generating subspace.
\begin{proposition}\label{prop:quotientnorm_transfer}
    Consider $x,y\in\R^n$ and two subspaces $\Vcal,\Wcal\subset \R^n$ of dimension $m<n$. Then,
    \begin{equation}
        \norm{[x]_{\Wcal}-[y]_{\Wcal}}_2 \leq \snorm{\Pi_{\Wcal}-\Pi_{\Vcal}}\cdot\norm{x-y}_2 + \norm{[x]_{\Vcal}-[y]_{\Vcal}}_2,
    \end{equation}
    where $\snorm{\Pi_{\Wcal}-\Pi_{\Vcal}} = \sin\br{\theta_{\max}(\Vcal,\Wcal)}$.
\end{proposition}
\begin{proof}
    The result follows from the triangle inequality and the definition of the spectral norm:
    \begin{align*}
    \norm{[x]_{\Wcal}-[y]_{\Wcal}}_2&= \norm{\Pi_{\Wcal^\perp}(x-y)}_2\\
    &= \norm{(\Pi_{\Wcal^\perp} - \Pi_{\Vcal^\perp})(x-y) +  \Pi_{\Vcal^\perp}(x-y)}_2\\
    &\leq \norm{(\Pi_{\Wcal^\perp} - \Pi_{\Vcal^\perp})(x-y)}_2 + \norm{\Pi_{\Vcal^\perp}(x-y)}_2\\
    &= \norm{(\Pi_{\Wcal} - \Pi_{\Vcal})(x-y)}_2 +  \norm{[x]_{\Vcal}-[y]_{\Vcal}}_2\\
    &\leq \snorm{\Pi_{\Wcal}-\Pi_{\Vcal}}\cdot\norm{x-y}_2 + \norm{[x]_{\Vcal}-[y]_{\Vcal}}_2.
\end{align*}
Furthermore, for a proof of $\snorm{\Pi_{\Wcal}-\Pi_{\Vcal}} = \sin\br{\theta_{\max}(\Vcal,\Wcal)}$ we refer to \citep{Zlatko2000}.
\end{proof}
The following proposition shows that the maximal principal angle between two transition laws can be upper bounded by the spectral norm difference of the transition laws.

\maxangle*

Before proceeding with the proof of Proposition~\ref{prop:maxangle_bound}, we need the following technical result.
\begin{proposition}\label{prop:smallest_singular_value}
    For any $P\in\Delta_{\SAcal}^{\Scal}$ the smallest singular value of $E-\gamma P$ satisfies
    \begin{equation}
        \sigma_{\min}\br{E-\gamma P} \geq \sqrt{|\Acal|/|\Scal|}(1-\gamma).
    \end{equation}
\end{proposition}
\begin{proof}
    The main idea of the proof is to use that $\sigma_{\min}(A) = \min_{x\neq 0}\norm{Ax}_2 / \norm{x}_2$ for any matrix $A\in\R^{n\times m}$. We first lower bound $\sigma_{\min}(I-\gamma P_a) = \sigma_{\min}\br{I-\gamma (P_a)^\top}$, where $P_a$ denotes the state transition matrix under action $a$. Let $x\in\R^{S}$, then we have
    \begin{align}
        \norm{\br{I-\gamma (P_a)^\top} x}_2 &\geq 1/\sqrt{|\Scal|}\norm{\br{I-\gamma (P_a)^\top} x}_1\\
        &\geq 1/\sqrt{|\Scal|}\br{\norm{x}_1 - \gamma \norm{ (P_a)^\top x}_1}\\
        &\geq (1-\gamma)/\sqrt{|\Scal|} \norm{x}_1 \geq (1-\gamma)/\sqrt{|\Scal|} \norm{x}_2,
    \end{align}
    where the third inequality follows from 
    \begin{equation}
        \norm{ (P_a)^\top x}_1 = \sum_{s} \abs{\sum_{s'} P_a(s|s')x(s')}\leq \sum_{s}\sum_{s'} P_a(s|s')\abs{x(s')} = \norm{x}_1.
    \end{equation}
    Hence, $\sigma_{\min}(I-\gamma P_a)\geq (1-\gamma)/\sqrt{|\Scal|}$. Therefore, we have for any $x\in\R^{\Scal}$ that
    \begin{align}
        \norm{(E-\gamma P)x}_2^2 = \sum_a \norm{(I-\gamma P_a)x}_2^2 \geq |\Acal|/|\Scal|(1-\gamma)^2 \norm{x}_2^2,
    \end{align}
    which yields the desired result.
\end{proof}
We are now ready to prove Proposition~\ref{prop:maxangle_bound}.
\begin{proof}[Proof of Proposition~\ref{prop:maxangle_bound}]
    The first principal angle $\theta_1(P,P')=0$ is zero as we always have $\ones\subseteq \Ucal$.
    The bound on the maximal angle follows from a well-known perturbation result for orthogonal projections. Namely, if $A, B\in\R^{n\times m}$ are matrices of the same rank and $\Pi_A, \Pi_B$ denote the orthogonal projections onto their column span, then we have \citep{ji1987perturbation}
    \begin{equation}
        \snorm{\Pi_A-\Pi_B} \leq \min\bc{\snorm{A^\dagger}, \snorm{B^\dagger}}\snorm{A-B},
    \end{equation}
    where $A^\dagger$ denotes the Moore-Penrose inverse \citep{penrose1955generalized}.
    Recall that $\Ucal_P = \ima(E-\gamma P)$, $\sin\br{\theta_{\max}(P,{P'})} = \snorm{\Pi_{\Ucal_P} - \Pi_{\Ucal_{P'}}}$, and by Proposition~\ref{prop:smallest_singular_value}
    \begin{equation}
        \snorm{(E-\gamma P)^\dagger} = \br{\sigma_{\min}\br{E-\gamma P}}^{-1} \leq \sqrt{|\Scal|/|\Acal|}H_\gamma.
    \end{equation}
    Therefore, we get
    \begin{equation}
        \sin\br{\theta_{\max}(P,{P'})} \leq \sqrt{|\Scal|/|\Acal|} \cdot \gamma \cdot H_\gamma \cdot \snorm{P-P'}.
    \end{equation}
\end{proof}

\section{Proof of claim in Example~\ref{ex1}}\label{app:sec:ex1}
We recall Example~\ref{ex1} from the main paper. 

\textbf{Example~\ref{ex1}.}
We consider a two-state, two-action MDP with $\Scal=\Acal=\bc{0,1}$, uniform initial state distribution, discount rate $\gamma = 0.9$, and Shannon entropy regularization $\polreg = - \Hcal$ (see Appendix~\ref{app:sec:regularizers}). Suppose the expert reward is $\rE(s,a) = \ind\{s=1\}$ and consider the transition laws, $P^0$ and $P^1$, defined by $P^0(0|s, a) = 0.75$ and $P^1(0|s,a) = 0.25 + \beta\cdot\ind\bc{s=0,a=0}$ for some $\beta\in[0,0.75]$. Also, consider the two experts $\muE_{P^0}=\RL_{P^0}(\rE)$ and $\muE_{P^1}=\RL_{P^1}(\rE)$, and suppose we recovered the reward $\rhat(s,a) = - \rE$. Then, the following holds: 1) We have $\mathsf{SubOpt}_{P^0}(\rhat, \muE_{P^0})=0$ and $\mathsf{SubOpt}_{P^1}(\rhat, \muE_{P^1})= \Ocal(\beta)$. That is, for small $\beta$, the reward $\rhat$ is a good solution to the IRL problem, as both experts are approximately optimal under $\rhat$. 2) The rank condition \eqref{eq:rank_cond} between $P^0$ and $P^1$ is satisfied for any $\beta>0$. 3) For a new transition law $P$ defined by $P(0|s,a)=\ind\bc{s=1,a=0}$, we have $\mathsf{SubOpt}_{P}(\rE, \RL_P(\rhat))\approx 4.81$, i.e. $\RL_P(\rhat)$ performs poorly under the experts' reward.

In the following we prove the claims 1. and 2., while 3. is computed via regularized dynamic programming \citep{geist2019theory}.\footnote{The code is openly accessible at \url{https://github.com/andrschl/irl_transferability}.} Furthermore, we illustrate the occupancy measure spaces corresponding to $P^0$ and $P^1$ for different $\beta$ in Figure~\ref{fig:occ_spaces}.
\begin{figure}[h]
\vspace{-0.2cm}
  \centering
  \includegraphics[width=0.9\linewidth]{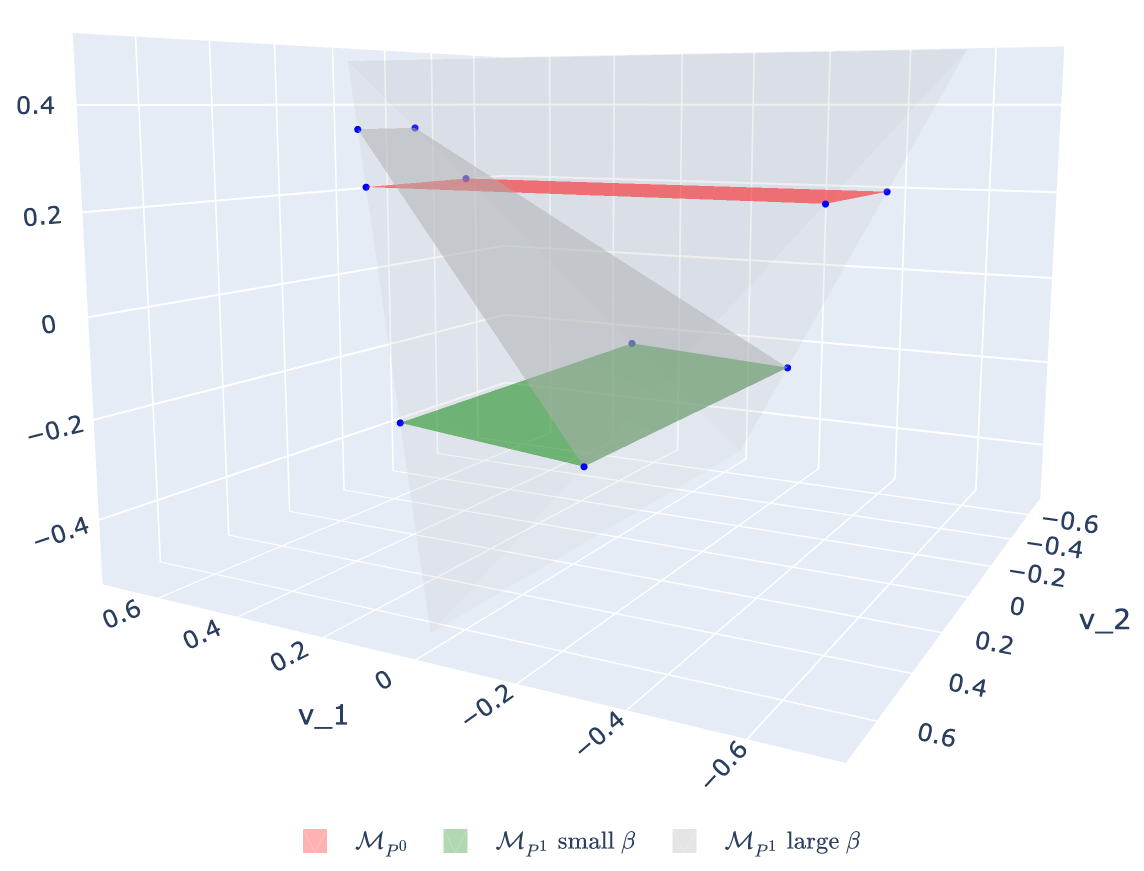}
  \caption{The set of occupancy measures $\Mcal_{P^0}$ and $\Mcal_{P^1}$ are illustrated in $\R^{\SAcal}/\ones\cong \ones^\perp$. For a two-state-two-action MDP, the set of occupancy measures is given by the intersection of a two-dimensional affine subspace (a plane in $\R^{\SAcal}/\ones$) with the probability simplex in $\R^4$ (a tetrahedron in $\R^{\SAcal}/\ones$). We see that for a small $\beta$, the sets $\Mcal_{P^0}$ and $\Mcal_{P^1}$ are approximately parallel. That is, the angle between their normal vectors, which span the potential shaping spaces $\Ucal_{P^0}$ and $\Ucal_{P^1}$, is small. In contrast, for a large $\beta$ the orientation of $\Mcal_{P^0}$ and $\Mcal_{P^1}$ is very different, resulting in a large angle between the corresponding normal vectors.}\label{fig:occ_spaces}
\end{figure}
\begin{enumerate}
    \item Consider the transition law $P'$ defined by $P'(0|s,a)=0.25$. First, we observe that while $\snorm{P^0 - P'}$ is large, the potential shaping spaces $\Ucal_{P^0}$ and $\Ucal_{P'}$ coincide. To see this note that we have $P'(\cdot|s,a) = P^0(\cdot|s,a) + \Delta$, where $\Delta=[-0.5, 0.5]^\top$. Hence, we have for any $x\in\R^2$ that 
    \begin{equation*}
        (E-\gamma P')x = (E-\gamma P^0)x - \gamma \ip{\Delta}{x} \ones_4 = (E-\gamma P^0)\br{x - \frac{\gamma \ip{\Delta}{x}}{1-\gamma} \ones_2},
    \end{equation*}
    where $\ones_n$ denotes the all-one vector in $\R^n$. Therefore, $\spn(E-\gamma P') = \spn(E-\gamma P^0)$. Moreover, we have 
    \begin{equation}
        \snorm{P^1 - P'} \leq \sqrt{\sum_{s,s',a} (P^1(s'|s,a)-P'(s'|s,a))^2} = \sqrt{2}\beta.
    \end{equation}
    In light of Propositions \ref{prop:unreg_subopt_bound}, \ref{prop:quotientnorm_transfer}, and \ref{prop:maxangle_bound}, this implies that
    \begin{align}
        \mathsf{SubOpt}_{P^1}(\rhat, \RL_{P^1}(\rE)) &\leq 2 \norm{[\rhat]_{\Ucal_{P^1}}-[\rE]_{\Ucal_{P^1}}}_2\\
        &\leq 2 \snorm{\Pi_{\Ucal_{P^1}}- \Pi_{\Ucal_{P'}}}\norm{\rhat -\rE}_2\\
        &\leq 2\gamma \cdot H_\gamma \cdot \snorm{P-P'}\leq 2\sqrt{2}\gamma \cdot H_\gamma \cdot \beta.
    \end{align}
    \item We need to show that $P^0$ and $P^1$ are satisfying the rank condition
    \begin{equation}
    \rank\br{\myvec{E-\gamma P^0, && E-\gamma P^1}} = 2|\Scal| - 1.
    \end{equation}
    By the same reasoning as above, we can equivalently show the rank condition for the transition laws ${P^0}', {P^1}'$ defined by ${P^0}'(0|s,a)=1$ and ${P^1}'(0|s,a)=\beta\cdot\ind\bc{s=0,a=0}$. To this end, we choose the matrix representation
    \begin{equation}
        E = \myvec{I\\ I} \quad \text{and} \quad P = \myvec{P_{a_0}\\ P_{a_1}},
    \end{equation}
    where $I\in\R^{|\Scal|\times |\Scal|}$ is the identity matrix and $P_{a_0}, P_{a_1}\in\R^{|\Scal|\times |\Scal|}$ are the state transition matrices corresponding to the actions $0,1$, respectively. Let $C = \myvec{E-\gamma {P^0}', && E-\gamma {P^1}'}$. We have
    \begin{equation}
    {P^0}'= \myvec{1 & 0\\ 1 & 0\\ 1 & 0\\ 1 & 0}  \quad \text{and} \quad {P^1}' =  \myvec{\beta & 1-\beta\\ 0 & 1\\ 0 & 1\\ 0 & 1},
    \end{equation}
    and
    \begin{equation}
    C = \myvec{1-\gamma & 0 & 1-\beta\gamma & -\gamma + \beta\gamma\\
            -\gamma & 1 & 0 & 1-\gamma\\
           1-\gamma & 0 & 1 & -\gamma\\
            -\gamma & 1 & 0 & 1-\gamma}.
    \end{equation}
    It's straightforward to see that the vector $\myvec{1&1&-1&-1}^\top$ lies in the kernel of $C$, but there is a $3\times 3$ submatrix with non-zero determinant:
    \begin{equation}
    \det\br{\myvec{1-\gamma & 0 & 1-\beta\gamma \\
            -\gamma & 1 & 0 \\
             1-\gamma & 0 & 1 }} = 1\cdot\bs{(1-\gamma) - (1-\gamma)(1-\beta\gamma)}=\beta\gamma(1-\gamma)>0.
    \end{equation}
    In other words, we have $\rank C = 3$ for any $\beta>0$.
\end{enumerate}

\section{Proof of Theorem~\ref{thm:global_transferability}}\label{app:sec:global_transferability}
\globaltransferability*
The proof of Theorem~\ref{thm:global_transferability} hinges on Lemma~\ref{lem:bregman_bound} and the following reward approximation result.
\begin{restatable}{lemma}{twoexpertidentifiability}
    \label{lem:two_expert_identifiability}
    Let $\norm{[\rE]_{\Ucal_{P^k}} - [\rhat]_{\Ucal_{P^k}}}_{2} \leq \epsilonbar$ for $k=0,1$. Then, if $\theta_2({P^0},{P^1})>0$, it holds that
    \begin{equation}\norm{[\rE]_{\ones} - [\rhat]_{\ones}}_{2} \leq \frac{\epsilonbar}{\sin\br{\theta_2({P^0},{P^1})/2}}. \end{equation}
\end{restatable}
\begin{proof}[Proof of Lemma~\ref{lem:two_expert_identifiability}]
    Throughout this proof, we will use the short-hand notation $\Ucal_k\defeq \Ucal_{P^k}$ for $k=0,1$. Recall that since $\ones\subseteq \Ucal_{0}\cap\Ucal_{1}$, we have $\theta_1(\Ucal_{0},\Ucal_{1})=0$ and by assumption we also have $\theta_2(\Ucal_{0},\Ucal_{1})>0$, which implies that $\Ucal_{0}\cap\Ucal_{1}=\ones$. Furthermore, since for $k=0,1$ we can rewrite $\R^{\SAcal}$ as the orthogonal sum $\R^{\SAcal}=\Ucal_{k}\cap\ones^\perp \oplus \Ucal_{k}^\perp \oplus \ones$, we can uniquely decompose $\rE - \rhat$ into $\rE - \rhat = x_k  + y_k + z$, where $x_k\in\Ucal_{k}\cap\ones^\perp$, $y_k\in\Ucal_{k}^\perp$, $z\in\ones$, for $k=0,1$. Then, it holds that $x_0  + y_0 = x_1 + y_1$. Since $\norm{[\rE]_{P^k} - [\rhat]_{P^k}}_{P^k, 2} = \norm{y_k}_2$, the Assumption of Lemma~\ref{lem:two_expert_identifiability} implies that $\norm{y_k}_2 \leq \epsilonbar$.
    For the $2$-distance between the equivalence classes $[\rE]_{\ones}$ and $[\rhat]_{\ones}$ the Pythagorean theorem implies that
    \begin{equation}\label{app:eq:identifiability1}
        \norm{[\rE]_{\ones} - [\rhat]_{\ones}}_{\ones, 2}^2 = \norm{x_0}_2^2 + \norm{y_0}^2 = \norm{x_1}_2^2 + \norm{y_1}_2^2 \leq \max_{\substack{u_k\in\Ucal_k\cap\ones^\perp, v_k\in\Ucal_k^\perp,\\
        \norm{u_k}_2=\norm{v_k}_2=1,\\
        \alpha_k\in\R_+, \beta_k\in[0,\epsilonbar], k=0,1,\\
        \alpha_0 u_0 + \beta_0 v_0 = \alpha_1 u_1 + \beta_1 v_1}} \alpha_0^2 + \beta_0^2,
    \end{equation}
    where the upper bound follows from $x_0+y_0=x_1+y_1$ and $\norm{y_k}_2 \leq \epsilonbar$.
    Next, we want to show that the maximum on the right-hand side of \eqref{app:eq:identifiability1} is achieved for $\beta_0 = \beta_1 = \epsilonbar$. To see this, note that taking inner products between $u_0$ and $u_1$, respectively, and the equation $\alpha_0 u_0 + \beta_0 v_0 = \alpha_1 u_1 + \beta_1 v_1$, we arrive at
    \begin{equation}
        \alpha_0 = \alpha_1 \ip{u_0}{u_1} + \beta_1 \ip{u_0}{v_1},\; \alpha_1 = \alpha_0 \ip{u_0}{u_1} + \beta_0 \ip{u_1}{v_0},\nonumber
    \end{equation}
    which is for any choice of $\beta_k, u_k, v_k, k=0,1$ an invertible linear system of equations for $\alpha_0,\alpha_1$ with the solutions
    \begin{equation}
        \alpha_0 = \dfrac{\beta_0 \ip{u_0}{u_1}\ip{u_1}{v_0}  + \beta_1\ip{u_0}{v_1}}{1-\ip{u_0}{u_1}^2},\; \alpha_1 = \dfrac{\beta_1 \ip{u_1}{u_0}\ip{u_0}{v_1}  + \beta_0\ip{u_1}{v_0}}{1-\ip{u_1}{u_0}^2} \label{app:eq:identifiability2}
    \end{equation}
    where $\ip{u_0}{u_1}<1$, due to $\Ucal_0\cap\Ucal_1\cap\ones^\perp = 0$. As the sign of $\ip{u_0}{u_1}\ip{u_1}{v_0}$ and $\ip{u_0}{v_1}$ can be chosen arbitrarily by an appropriate choice of $v_0, v_1$, the objective in the right-hand-side of \eqref{app:eq:identifiability1} is increasing in $\beta_0,\beta_1$ and hence the maximum is achieved for $\beta_0 = \beta_1 = \epsilonbar$ and $\alpha\defeq\alpha_0 = \alpha_1 = \frac{\epsilonbar \ip{u_0}{v_1}}{1-\ip{u_0}{u_1}}$. Therefore, it holds that
    \begin{align*}
        \norm{[\rE]_{\ones} - [\rhat]_{\ones}}_{\ones, 2}^2 &\leq \max_{\substack{u_k\in\Ucal_k\cap\ones^\perp, v_1\in\Ucal_1^\perp,\\
        \norm{u_k}_2=\norm{v_1}_2=1, k=0,1}} \epsilonbar^2\bs{1 + \br{\dfrac{\ip{u_0}{v_1}}{1-\ip{u_0}{u_1}}}^2}\\
        &\stackrel{(i)}{=}\max_{\substack{u_0\in\Ucal_0\cap\ones^\perp, \\
        \norm{u_0}_2=1}} \epsilonbar^2\bs{1 + \br{\dfrac{\max_{v_1\in\Ucal_1^\perp, \norm{v_1}_2=1}\ip{u_0}{v_1}}{1-\max_{u_1\in\Ucal_1\cap\ones^\perp, \norm{u_1}_2=1}\ip{u_0}{u_1}}}^2}\\
        &\stackrel{(ii)}{=} \max_{\substack{u_0\in\Ucal_0\cap\ones^\perp, \\
        \norm{u_0}_2=1}} \epsilonbar^2\bs{1 + \br{\dfrac{\norm{\Pi_{\Ucal_1^\perp}u_0}_2}{1-\norm{\Pi_{\Ucal_1\cap\ones^\perp} u_0}_2}}^2}\\
        &\stackrel{(iii)}{=}\max_{\substack{u_0\in\Ucal_0\cap\ones^\perp, \\
        \norm{u_0}_2=1}} \epsilonbar^2\bs{1 + \br{\dfrac{\sqrt{1-\norm{\Pi_{\Ucal_1}u_0}_2^2}}{1-\norm{\Pi_{\Ucal_1\cap\ones^\perp} u_0}_2}}^2}\\
        &\stackrel{(iv)}{=}\max_{\substack{u_0\in\Ucal_0\cap\ones^\perp, \\
        \norm{u_0}_2=1}} \epsilonbar^2\bs{1 + \br{\dfrac{\sqrt{1-\norm{\Pi_{\Ucal_1\cap\ones^\perp}u_0}_2^2}}{1-\norm{\Pi_{\Ucal_1\cap\ones^\perp} u_0}_2}}^2}\\
        &\stackrel{(v)}{=}\max_{\substack{u_0\in\Ucal_0\cap\ones^\perp, \\
        \norm{u_0}_2=1}} \epsilonbar^2\bs{1 + \dfrac{1+\norm{\Pi_{\Ucal_1\cap\ones^\perp}u_0}_2}{1-\norm{\Pi_{\Ucal_1\cap\ones^\perp} u_0}_2}}\\
        &\stackrel{(vi)}{=} \epsilonbar^2\bs{1 + \dfrac{1+\cos\br{\theta_2(\Ucal_0, \Ucal_1)}}{1-\cos\br{\theta_2(\Ucal_0, \Ucal_1)}}}\\
        &\stackrel{(vii)}{=} \epsilonbar^2\dfrac{2}{1-\cos\br{\theta_2(\Ucal_0, \Ucal_1)}}\\
        &\stackrel{(viii)}{=} \dfrac{\epsilonbar^2}{\sin\br{\theta_2(\Ucal_0, \Ucal_1)/2}^2}.
    \end{align*}
    Here, we took the maximum over $u_1, v_1$ in $(i)$, we used that $\max_{v\in \Vcal, \norm{v}_2=1}\ip{v}{u} = \norm{\Pi_{\Vcal} u}_2$ in $(ii)$, and $(iii)$ follows from the Pythagorean theorem. Furthermore, $(iv)$ follows from $u_0\in\ones^\perp$ and $(v)$ from simplifying. In $(vi)$ we then again use $\max_{v\in \Vcal, \norm{v}_2=1}\ip{v}{u} = \norm{\Pi_{\Vcal} u}_2$, the definition of the second principal angle (Definition~\ref{def:principal_angles}), and the fact that the first principal vectors lie in $\ones$. Lastly, $(vii)$ follows from simplifying and $(viii)$ from $\sin(x/2)^2 = (1-\cos x)/2$.
\end{proof}

\begin{proof}[Proof of Theorem~\ref{thm:global_transferability}]
    As mentioned in the proof sketch in the main paper, it follows from the lower bound in Lemma~\ref{lem:bregman_bound} that $\norm{[\rE]_{\Ucal_{P^k}} - [\rhat]_{\Ucal_{P^k}}}_2\leq \sqrt{2\epsilonhat/\sigma_{\Rcal}}$. In light of Lemma~\ref{lem:two_expert_identifiability}, this implies that for any $P\in\Delta_{\SAcal}^{\Scal}$ we have
    \begin{equation}
        \norm{[\rE]_{\Ucal_P} - [\rhat]_{\Ucal_P}}_{2} \leq \norm{[\rE]_{\ones} - [\rhat]_{\ones}}_{2} \leq \dfrac{\sqrt{2\epsilonhat/\sigma_{\Rcal}}}{\sin\br{\theta_2({P^0},{P^1})/2}}.
    \end{equation}
Hence, applying the upper bound in Lemma~\ref{lem:bregman_bound} yields
\begin{equation}
    \mathsf{SubOpt}_P(\rE, \RL_P(\rhat)) \leq \dfrac{1}{2\eta} \norm{[\rE]_{\Ucal_P} - [\rhat]_{\Ucal_P}}_{2}^2 \leq \dfrac{\epsilonhat}{\eta \sigma_{\Rcal}\sin\br{\theta_2({P^0},{P^1})/2}^2}.
\end{equation}
\end{proof}

\section{Proof of Theorem~\ref{thm:local_transferability}}\label{app:sec:local_transferability}
\localtransferability*

\begin{proof}
    Similar to Theorem~\ref{thm:global_transferability}, it follows from Lemma~\ref{lem:bregman_bound} that $\norm{[\rE]_{\Ucal_{P^0}} - [\rhat]_{\Ucal_{P^0}}}_2\leq \sqrt{2\epsilonhat/\sigma_{\Rcal}}$. By Proposition~\ref{prop:quotientnorm_transfer}, we then have that
    \begin{align}
        \norm{[\rE]_{\Ucal_{P}} - [\rhat]_{\Ucal_{P}}}_2 &\leq \sin\br{\theta_{\max}(P, {P^0})} \norm{\rE-\rhat}_2 + \norm{[\rE]_{\Ucal_{P^0}} - [\rhat]_{\Ucal_{P^0}}}_2\\
        &\leq \sin\br{\theta_{\max}(P, {P^0})} \DR + \sqrt{2\epsilonhat/\sigma_{\Rcal}}.
    \end{align}
    Hence, applying Lemma~\ref{lem:bregman_bound} again yields
    \begin{align}
        \mathsf{SubOpt}_P\br{\rE, \RL_P(\rhat)} &\leq \dfrac{1}{2\eta} \norm{[\rE]_{\Ucal_{P}} - [\rhat]_{\Ucal_{P}}}_2^2\\
        &\leq \dfrac{\br{\DR\sin\br{\theta_{\max}(P, {P^0})} + \sqrt{2\epsilonhat/\sigma_{\Rcal}}}^2}{2\eta}\\
        &\leq \dfrac{2\max\bc{\DR^2\sin\br{\theta_{\max}(P, {P^0})}^2 , 2\epsilonhat/\sigma_{\Rcal}}}{\eta}.
    \end{align}
\end{proof}

\section{Estimating principal angles}\label{app:sec:estimating_principal_angles}
Consider two full rank matrices $A, B\in\R^{n\times m}$ and let the columns of $U_{A}, U_{B}\in\R^{n\times m}$ form an orthonormal basis of $\Vcal = \ima A$ and $\Wcal = \ima B$, respectively. Then, as discussed by \citep{ji1987perturbation} we have
\begin{equation}
    \sigma_i = \cos(\theta_{i}(\Vcal, \Wcal)), i=1,\hdots,m,
\end{equation}
where $1 \geq \sigma_1 \geq \hdots \geq \sigma_m\geq 0$ denote the singular values of $U_{A}^\top U_{B}$ sorted in decreasing order. Hence, given the transition matrices $P^0, P^1$, we can compute the principle angles $\theta_i(P^0, P^1)$ by first computing orthonormal bases for the column spans of $E-\gamma P_i, i=1,2$, and then computing the singular values as described above.

Now, suppose that $\hat{P}^0,\hat{P}^1$ are empirical estimates of $P^0, P^1$, then we have by \citep[Theorem 3.1]{ji1987perturbation} the following perturbation result
\begin{align}
    \max_i\abs{\sin(\theta_i(P^0, P^1))-\sin(\theta_i(\hat P^0, \hat P^1))} &\leq \norm{\Pi_{\Ucal_{P^0}}-\Pi_{\Ucal_{\hat P^0}}}+\norm{\Pi_{\Ucal_{P^1}}-\Pi_{\Ucal_{\hat P^1}}}\\
    &\leq \gamma  H_\gamma\sqrt{|\Scal|/|\Acal|} \br{\snorm{P^0-\hat P^0}+\snorm{P^1-\hat P^1}},
\end{align}
where the last inequality follows from Propositions~\ref{prop:quotientnorm_transfer} and \ref{prop:maxangle_bound}. Hence, we can estimate $\sin  \theta_i(P^0, P^1)$ up to an error of $\Ocal\br{\max\bc{\snorm{P^0-\hat P^0}, \snorm{P^1-\hat P^1}}}$.

\section{Proof of Theorem~\ref{thm:convergence}}\label{app:sec:irl_convergence}
\irlconvergence*
\noindent The proof of Theorem~\ref{thm:convergence} is inspired by \citep[Theorem~2]{syed2007game}. However, in contrast to \citet{syed2007game}, we consider the regularized problem with multiple experts, we use the suboptimality as the convergence metric, and we use a projected gradient descent update (instead of multiplicative weights). The proof hinges on Hoeffding's inequality and a regret bound for online gradient descent, which are provided in Theorem~\ref{thm:hoeffding} and \ref{thm:ogd} below.

\begin{restatable}[Hoeffding's inequality \citep{hoeffding1963}]{theorem}{hoeffding}
    \label{thm:hoeffding}%
    Let $X_0,\hdots,X_{M-1}$ be independent random variables with $X_l\in[a,b]$ and let $S_M:=X_0+\hdots+X_{M-1}$. Then,
    \begin{equation}
        \Pr\br{\abs{S_M - \E S_M} \geq c}\leq 2\exp\br{-\dfrac{2c^2}{M(b-a)^2}}.
    \end{equation}
\end{restatable}
\begin{restatable}[Online gradient descent \citep{zinkevich2003online}]{theorem}{ogd}
    \label{thm:ogd}%
    Consider some bounded closed convex set $\Xcal\subset \R^n$ with $D\defeq \max_{x,x'\in\Xcal} \norm{x-x'}_2$. Moreover, let $\Pi_{\Xcal}:\R^n\to\Xcal$ be the orthogonal projection onto $\Xcal$. For any sequence of convex differentiable functions $f_0,\hdots, f_{T-1} : \Xcal \to \R$ satisfying $\max_{x\in\Xcal} \norm{\nabla f_t(x)}_2\leq G$, the online projected gradient descent update 
    \begin{equation}
        x_{t+1}\leftarrow\Pi_{\Xcal}\br{x_t-\alpha \nabla f_t(x_t)},
    \end{equation}
    with step-size $\alpha=D/(G\sqrt{T})$ satisfies
    \begin{equation}
        \sum_{t=0}^{T-1} f_t(x_t) - \min_{x^*\in\Xcal}  \sum_{t=0}^{T-1} f_t(x^*) \leq DG\sqrt{T}.
    \end{equation}
\end{restatable}

\begin{proof}[Proof of Theorem~\ref{thm:convergence}]
The proof is in three steps. First, we use Hoeffding's inequality to prove concentration of the empirical occupancy measures around the true occupancy measures. Then, we use the union bound to upper bound the probability that any of our bounds fails to hold. Finally, we prove the convergence rate of Algorithm~\ref{alg:multi_expert_irl} using the regret bound in Theorem~\ref{thm:ogd}.

\emph{Step 1:} 
Let $\Dcal=\bc{(s_0,a_0,\hdots, s_{H-1}, a_{H-1})}_{i=0}^{N-1}$ be sampled from some policy $\pi^\mu$ and recall that the corresponding empirical occupancy measure is defined as
\begin{equation*}
     \muhat_{\Dcal}(s,a) = \dfrac{1-\gamma}{N} \sum_{i=0}^{N-1} \sum_{t=0}^{H-1} \gamma^t \ind\{s_t^i=s, a_t^i=a\}.
\end{equation*}
It will be convenient to define the truncated occupancy measure
\begin{equation*}
     \mu_H(s,a) = (1-\gamma)\sum_{t=0}^{H-1} \gamma^t \Pp_{\nu_0}^{\pi^\mu}\{s_t^i=s, a_t^i=a\}.
\end{equation*}
For $K$ data sets $\Dcal_1, \hdots, \Dcal_K$ sampled from $\pi^{\mu_k}$ we then have
\begin{align*}
    \max_{r\in\Rcal} \sum_{k=0}^{K-1}\ip{r}{\mu_k - \muhat_{\Dcal_k}} &\stackrel{(i)}{\leq} \max_{r\in\Rcal}\norm{r}_1\norm{\sum_{k=0}^{K-1} \br{\mu_k - \muhat_{\Dcal_k}}}_\infty \stackrel{(ii)}{\leq} \norm{\sum_{k=0}^{K-1} \br{\mu_k - \muhat_{\Dcal_k}}}_\infty\\
    &\stackrel{(iii)}{\leq} \underbrace{\norm{\sum_{k=0}^{K-1}\br{\mu_k - \mu_{H, k}}}_\infty}_{I_1} + \underbrace{\norm{\sum_{k=0}^{K-1}\br{\mu_{H,k} - \muhat_{\Dcal_k}}}_\infty}_{I_2},
\end{align*}
where $(i)$ follows from Hölder's inequality, $(ii)$ from our definition of $\Rcal$ as the 1-norm ball, and $(iii)$ from the triangle inequality. Since $\norm{\mu - \mu_H}_\infty \leq \gamma^H$, we have $I_1\leq \gamma^H K$. Moreover, applying Hoeffding's inequality to the $M=KN$ independent random variables
\begin{equation*}
    X_{kN + i} = \dfrac{1-\gamma}{N} \sum_{t=0}^{H-1} \gamma^t \ind\{ s_t^{k,i}=s, a_t^{k,i}=a \}, \; i\in[N], k\in[K],
\end{equation*}
with $X_i\in[0, 1/N]$, we arrive at
\begin{equation*}
    \Pr\br{\abs{S_M - \E S_M} \geq \varepsilon_{\text{stat}}/2}=\Pr\br{\abs{\sum_{k=0}^{K-1} \muhat_{\Dcal_k}(s,a)-\mu_{K,k}(s,a)} \geq \varepsilon_{\text{stat}}/2} \leq 2\exp\br{-\dfrac{\varepsilon_{\text{stat}}^2 N}{2K}}.
\end{equation*}
Hence, applying the union bound over all $|\Scal||\Acal|$ components of the occupancy measure yields
\begin{equation*}
    \Pr(I_2 < \varepsilon_{\text{stat}}/2) = 1- \Pr(I_2 \geq \varepsilon_{\text{stat}}/2) \geq 1 - 2|\Scal||\Acal|\exp\br{-\dfrac{\varepsilon_{\text{stat}}^2 N}{2K}}.
\end{equation*}
Therefore, to ensure that with probability at least $1-\delta_{\text{stat}}$ it holds that 
\begin{equation}\max_{r\in\Rcal} \sum_{k=0}^{K-1}\ip{r}{\mu_k - \muhat_{\Dcal_k}}\leq \varepsilon_{\text{stat}},\end{equation}
it suffices to choose 
\begin{equation*}
    N\geq  \dfrac{2K\log\br{2|\Scal||\Acal|/\delta_{\text{stat}}}}{\varepsilon_{\text{stat}}^2} \quad \text{and} \quad H\geq \dfrac{\log\br{2K/\varepsilon_{\text{stat}}}}{\log(1/\gamma)}.
\end{equation*}
This concentration result applies to both empirical occupancy measures generated from the expert data sets $\DE_k$, as well as the data sets $\Dcal_{k,t}$ generated by Algorithm~\ref{alg:multi_expert_irl}.

\emph{Step 2:}
When analyzing Algorithm~\ref{alg:multi_expert_irl} there are three sources of stochasticity. The first two are due to the randomness in the data sets $\DE_k$ and $\Dcal_{k,t}$, and the third is due to the randomness in the forward RL algorithm, $\ARL^{\varepsilon_{\text{opt}}, \delta_{\text{opt}}}_{P^k}$, that upon a query with the reward $r_t$ outputs a policy $\pi_{k,t}$ such that with probability at least $1-\delta_{\text{opt}}$ it holds $\mathsf{SubOpt}_{P^k}(r_t, \mu^{\pi_{k,t}})\leq \varepsilon_{\text{opt}}$. Let's denote the event that $\max_{r\in\Rcal} \sum_{k=0}^{K-1}\ip{r}{\muE_{P^k} - \muEhatk} > \varepsilon_{\text{stat}, \textsf{E}}$ by $\Ecal_{\text{stat}, \textsf{E}}$, the event that $\max_{r\in\Rcal} \sum_{k=0}^{K-1}\ip{r}{\mu^{\pi_{k,t}} - \muhat_{\Dcal_{k,t}}} > \varepsilon_{\text{stat}}$ by $\Ecal_{\text{stat}, t}$, and the event that $\mathsf{SubOpt}_{P^k}(r_t, \mu^{\pi_{k,t}})> \varepsilon_{\text{opt}}$ by $\Ecal_{\text{opt}, k, t}$. Moreover, let us assume that $\Ecal_{\text{stat}, \textsf{E}}$ happens with probability at most $\delta_{\text{stat}, \textsf{E}}$, $\Ecal_{\text{stat}, t}$ happens with probability at most $\delta_{\text{stat}}$, and $\Ecal_{\text{opt},k, t}$ happens with probability at most $\delta_{\text{opt}}$. By union bound, the probability of the event
\begin{equation}
    \Fcal \defeq \neg\Ecal_{\text{stat}, \textsf{E}} \wedge \bigwedge_{t=0}^{T-1} \neg \Ecal_{\text{stat}, t}\wedge \bigwedge_{t=0}^{T-1}\bigwedge_{k=0}^{K-1} \neg \Ecal_{\text{opt}, k, t},
\end{equation}
that none of the above events happens is lower bounded by
\begin{align}
    \Pr\br{\Fcal} &= 1 - \Pr\br{\Ecal_{\text{stat}, \textsf{E}} \lor \bigvee_{t=0}^{T-1} \Ecal_{\text{stat}, t}\lor \bigvee_{t=0}^{T-1}\bigvee_{k=0}^{K-1} \Ecal_{\text{opt}, k, t}} \\
    &\geq 1 - \br{\Pr\br{\Ecal_{\text{stat}, \textsf{E}}} + \sum_{t=0}^{T-1}\Pr\br{\Ecal_{\text{stat}, t}}  + \sum_{t=0}^{T-1}\sum_{k=0}^{K-1} \Pr\br{\Ecal_{\text{opt}, k, t}}}\\
    &\geq 1 - \br{\delta_{\text{stat}, \textsf{E}} + T\delta_{\text{stat}} + KT \delta_{\text{opt}}}.
\end{align}
Hence, to ensure that $\Fcal$ happens with probability at least $1-\deltahat$, it suffices to choose
\begin{align}
    N&\geq  \dfrac{2K\log\br{6|\Scal||\Acal|/\deltahat}}{\varepsilon_{\text{stat}, \textsf{E}}^2} \quad \text{and} \quad H\geq \dfrac{\log\br{2K/\varepsilon_{\text{stat}, \textsf{E}}}}{\log(1/\gamma)},\\
    N_t&\geq  \dfrac{2K\log\br{6T|\Scal||\Acal|/\deltahat}}{\varepsilon_{\text{stat}}^2}\quad \text{and} \quad \delta_{\text{opt}} = \dfrac{\deltahat}{3KT}.
\end{align}

\emph{Step 3:}
    Note that we can bound $\norm{g_t}_2 \leq \norm{g_t}_1 \leq \sum_{k=0}^{K-1} \norm{\muEhatk}_1 + \norm{\muhat_{k,t}}_1 \leq 2K \eqdef G$ and the diameter of $\Rcal$ is $D=2$. Hence, given that event $\Fcal$ happens, we can bound the suboptimalities of the $K$ experts under the reward, $\rhat$, recovered by Algorithm~\ref{alg:multi_expert_irl} with stepsize $\alpha=D/(G\sqrt{T})$ as follows
    \begin{align}
        &\sum_{k=0}^{K-1}\mathsf{SubOpt}_{P^k}(\rhat, \muE_{P^k})\nonumber\\
        =& \sum_{k=0}^{K-1} \bs{\max_{\mu\in\Mcal_{P^k}} \ip{\rhat}{\mu - \muE_{P^k}} - \occreg(\mu) + \occreg(\muE_{P^k})}\nonumber\\
        \stackrel{(i)}{\leq}& \varepsilon_{\text{stat}, \textsf{E}} + \sum_{k=0}^{K-1} \bs{\max_{\mu\in\Mcal_{P^k}} \ip{\rhat}{\mu - \muEhatk} - \occreg(\mu) + \occreg(\muE_{P^k})}\nonumber\\
        \stackrel{(ii)}{\leq}& \varepsilon_{\text{stat}, \textsf{E}} + \sum_{k=0}^{K-1} \dfrac{1}{T}\sum_{t=0}^{T-1} \bs{\max_{\mu\in\Mcal_{P^k}} \ip{r_t}{\mu - \muEhatk} - \occreg(\mu) + \occreg(\muE_{P^k})}\nonumber\\
        =&  \varepsilon_{\text{stat}, \textsf{E}} + \dfrac{1}{T}\sum_{t=0}^{T-1} \sum_{k=0}^{K-1} \bs{\max_{\mu\in\Mcal_{P^k}} \ip{r_t}{\mu - \muEhatk} - \occreg(\mu) + \occreg(\muE_{P^k})}\nonumber\\
        \stackrel{(iii)}{\leq}&   \varepsilon_{\text{stat}, \textsf{E}} + K\varepsilon_{\text{opt}} + \dfrac{1}{T}\sum_{t=0}^{T-1} \sum_{k=0}^{K-1} \bs{\ip{r_t}{\mu_{k,t}- \muEhatk} - \occreg(\mu_{k,t}) + \occreg(\muE_{P^k})}\nonumber\\
        \stackrel{(iv)}{\leq}&  \varepsilon_{\text{stat}, \textsf{E}} + K\varepsilon_{\text{opt}} + \varepsilon_{\text{stat}} + \dfrac{1}{T}\sum_{t=0}^{T-1} \sum_{k=0}^{K-1} \bs{\ip{r_t}{\muhat_{\Dcal_{k,t}}- \muEhatk} - \occreg(\mu_{k,t}) + \occreg(\muE_{P^k})}\nonumber\\
        \stackrel{(v)}{\leq}& \varepsilon_{\text{stat}, \textsf{E}} + K\varepsilon_{\text{opt}} + \varepsilon_{\text{stat}} +  \dfrac{DG}{\sqrt{T}} + \min_{r\in\Rcal}\dfrac{1}{T}\sum_{t=0}^{T-1} \sum_{k=0}^{K-1} \bs{\ip{r}{\muhat_{\Dcal_{k,t}}- \muEhatk} - \occreg(\mu_{k,t}) + \occreg(\muE_{P^k})}\nonumber\\
        \stackrel{(vi)}{\leq}&  2\varepsilon_{\text{stat}, \textsf{E}} + K\varepsilon_{\text{opt}} + 2\varepsilon_{\text{stat}}+ \dfrac{DG}{\sqrt{T}} + \min_{r\in\Rcal} \dfrac{1}{T}\sum_{t=0}^{T-1}\sum_{k=0}^{K-1}\bs{ \ip{r}{\mu_{k,t}- \muE_{P^k}} - \occreg(\mu_{k,t}) + \occreg(\muE_{P^k})} \nonumber\\
        \stackrel{(vii)}{\leq}& 2\varepsilon_{\text{stat}, \textsf{E}} + K\varepsilon_{\text{opt}} + 2\varepsilon_{\text{stat}}+  \dfrac{DG}{\sqrt{T}} \nonumber\\
        & \hspace{1.05cm}  +\underbrace{\min_{r\in\Rcal} \sum_{k=0}^{K-1} \bs{\ip{r}{\mubar_k- \muE_{P^k}} - \occreg(\mubar_k) + \occreg(\muE_{P^k})}}_{\leq 0}, \quad \text{with } \mubar_k\defeq \dfrac{1}{T}\sum_{t=0}^{T-1}\mu_{k,t},\nonumber\\
        \stackrel{(viii)}{\leq}& 2\varepsilon_{\text{stat}, \textsf{E}} + K\varepsilon_{\text{opt}} + 2\varepsilon_{\text{stat}}+  \dfrac{DG}{\sqrt{T}} = 2\varepsilon_{\text{stat}, \textsf{E}} + K\varepsilon_{\text{opt}} + 2\varepsilon_{\text{stat}}+  \dfrac{4K}{\sqrt{T}}.
    \end{align}
    Here, the inequalities $(i),(iv)$, and $(vi)$ follow from the concentration bound established in step 1. Moreover, inequality $(ii)$ holds since $\rhat\mapsto \max_{\mu\in\Mcal_{P^k}} \ip{\rhat}{\mu - \muE_{P^k}} - \occreg(\mu) + \occreg(\muE_{P^k})$ is the pointwise maximum of affine functions and therefore convex. Furthermore, $(iii)$ follows from $\varepsilon_{\text{opt}}$-optimality of $\mu_{k,t}$, $(v)$ from Theorem~\ref{thm:ogd}, and $(vii)$ from concavity of the mapping $\mu_{k,t} \mapsto \ip{r}{\mu_{k,t}- \muE_{P^k}} - \occreg(\mu_{k,t})$. Finally, $(viii)$ holds because all experts are optimal for the reward $\rE$. In conclusion, to ensure that with probability at least $1-\deltahat$ it holds that $\mathsf{SubOpt}_{P^k}(\rhat, \muE_{P^k})\leq \sum_{k=0}^{K-1}\mathsf{SubOpt}_{P^k}(\rhat, \muE_{P^k})\leq \epsilonhat$ it suffices to choose $T=\frac{256K^2}{\epsilonhat^2}$, $\alpha=\frac{\epsilonhat}{16K^2}$, $N = \frac{128K\log\br{6 |\Scal||\Acal|/\deltahat}}{\epsilonhat^2}$, $H = H_t = \frac{\log\br{16K/\epsilonhat}}{\log(1/\gamma)}$, $N_t = \frac{128K\log\br{1536 K^2 |\Scal||\Acal|/(\deltahat\epsilonhat^2)}}{\epsilonhat^2}$, $\delta_{\text{opt}} = \frac{\deltahat\epsilonhat^2}{768K^3}$ , $\varepsilon_{\text{opt}}=\frac{\epsilonhat}{4K}$.
\end{proof}

\section{Suboptimal experts}\label{app:sec:suboptimal_experts}
In our problem formulation, we assumed that the $K$ experts are optimal with respect to $\rE$, i.e. $\muE_{P^k} = \RL_{P^k}(\rE)$ for $k=0,\hdots, K-1$. This assumption can be weakened by requiring that
\begin{equation}
    \max_{r\in\Rcal}\abs{J(r, \muE_{P^k})-J(r,\RL_{P^k}(\rE))}\leq \varepsilon_{\text{mis}},
\end{equation}
where $\varepsilon_{\text{mis}}$ is some misspecification error. The transferability results in Theorem~\ref{thm:global_transferability} and \ref{thm:local_transferability} still apply whenever we recover a reward $\rhat$ such that $\mathsf{SubOpt}_{P^k}(\rhat, \RL_{P^k}(\rE))\leq \epsilonhat$. Moreover, with a straightforward modification of the proof of Theorem~\ref{thm:convergence}, it follows that with high probability Algorithm~\ref{alg:multi_expert_irl} recovers a reward $\rhat$ such that $\mathsf{SubOpt}_{P^k}(\rhat, \RL_{P^k}(\rE))\leq \epsilonhat + 2K\varepsilon_{\text{mis}}$. Hence, our end-to-end transferability guarantees apply with $\epsilonhat\leftarrow \epsilonhat + 2K\varepsilon_{\text{mis}}$. However, $\varepsilon_{\text{mis}}$ cannot be further reduced by collecting more samples from the expert or MDP.

\section{Experimental details}\label{app:sec:experiments}
\paragraph{Setup}
To validate our results experimentally, we are using a stochastic adaption of the \texttt{WindyGridworld} environment \citep{sutton2018reinforcement}.\footnote{All our experiments were carried out -- within a day -- on a MacBook Pro with an Apple M1 Pro chip and 32 GB of RAM. } In particular, we consider a 6x6 grid with 4 actions (Up, Down, Left, Right), a wind direction (North, East, South, West), and a wind strength $\beta\in[0,1]$. When the agent takes an action, with probability $(1-\beta)$, it moves to the intended grid cell, and with probability $\beta$, the wind pushes the agent one step further in the direction of the wind. This means that the transition law is a convex combination of two laws: $(1-\beta)P^{\text{Gridworld}} + \beta P^{\text{Wind}}$, where $P^{\text{Gridworld}}$ and $P^{\text{Wind}}$ represent the transition laws for a deterministic \texttt{Gridworld} and a deterministic \texttt{WindyGridworld}. For our experiments, we then consider the pairs of expert transition laws $P_{\beta}^0 = (1-\beta)P^{\text{Gridworld}} + \beta P^{\text{North}}$ and $P_{\beta}^1 = (1-\beta)P^{\text{Gridworld}} + \beta P^{\text{East}}$ with $\beta$ in $\{0.01,0.1,0.5,1.0\}$. As shown in Figure~\ref{fig:experiment}(a), the second principal angle between $P_{\beta}^0$ and $P_{\beta}^1$, calculated using a singular value decomposition \citep{knyazev2002principal}, increases as the wind strength $\beta$ increases.

\paragraph{Inverse reinforcement learning} 
We observed that under a small second principal angle, the recovered reward heavily depends on both the expert reward and the reward initialization. Hence, we sample 10 independent expert rewards, each generated by first sampling a random set of 10 state-action pairs and then randomly assigning a reward of $\pm 1$. Using Shannon entropy regularization with $\tau=0.3$, we then use soft policy iteration to get expert policies for each combination of expert reward and wind strength $\beta$. For each of these expert policies, we then generate expert data sets with $\NE\in\{10^3, 10^4, 10^5, 10^6\}$ trajectories of length $H=100$. Next, we run Algorithm~\ref{alg:multi_expert_irl}, with soft policy iteration as a subroutine, for $30'000$ iterations, where rewards are initialized by sampling from a standard normal distribution. As a reward class, we choose the $\norm{\cdot}_1$-ball with radius $10^3$ (essentially unbounded), as a stepsize $\alpha = 0.05$ for the first $15'000$ iterations and $\alpha = 0.005$ for the second half. Moreover, we sample $N=100$ new trajectories of horizon $H=100$ at each gradient step. Figure~\ref{fig:experiment}(b) illustrates the distances between the recovered $\rhat$ and the experts' reward $\rE$, measured in $\R^{\SAcal}/\ones$. It is evident that the recovered reward gets closer to the experts' reward as the number of expert demonstrations increases. Moreover, we observe that the recovered reward is closer to the experts' reward when the second principal angle between the experts is larger, as expected from Lemma~\ref{lem:two_expert_identifiability}.\looseness-1

\begin{figure}[h]
  \centering
  \includegraphics[width=0.9\linewidth]{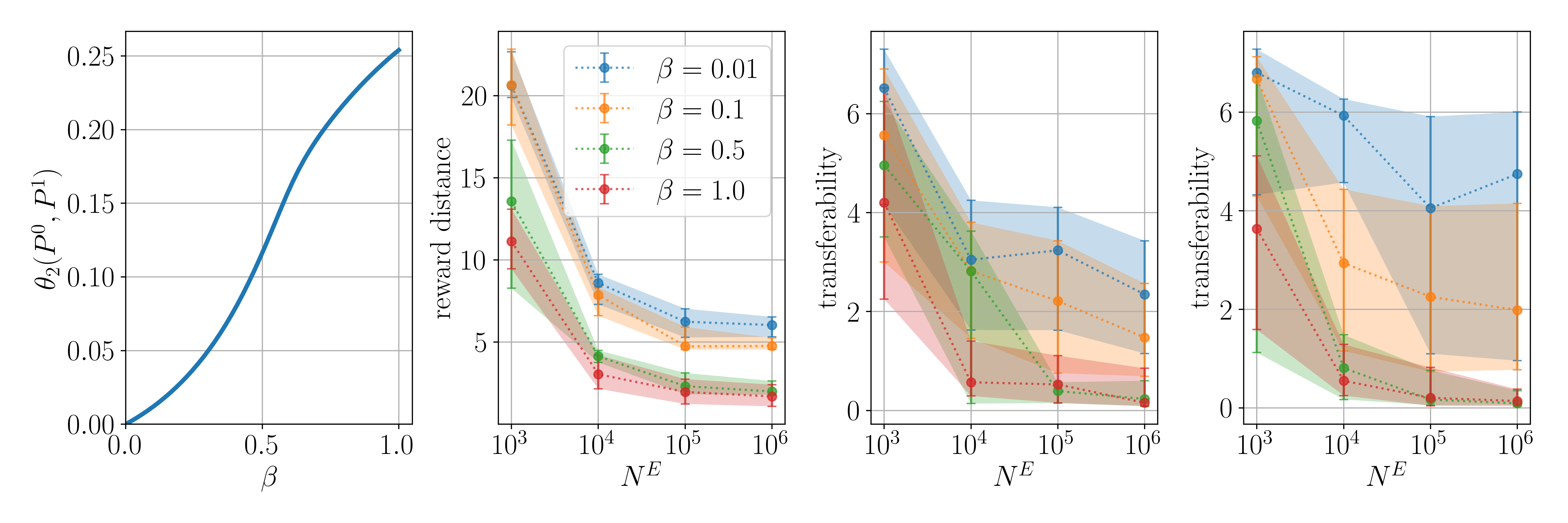}
  \vspace{-0.5cm}
    \begin{flushleft}\hspace{2.2cm}(\textit{a})  \hspace{2.8cm}(\textit{b})  \hspace{2.8cm}(\textit{c})  \hspace{2.8cm}(\textit{d})\end{flushleft}\vspace{-0.1cm}
  \caption{(\textit{a}) shows the second principal angle between $P_{\beta}^0$ and $P_{\beta}^1$ for varying wind strength $\beta$. Furthermore, (\textit{b}) shows the distance between $\rhat$ and $\rE$ in $\R^{\SAcal}/\ones$ for a varying number of expert demonstrations $\NE$ and wind strength $\beta$. Moreover, (\textit{c}) and (\textit{d}) show the transferability to $P^{\text{South}}$ and $P^{\text{Shifted}}$ in terms of $\mathsf{SubOpt}_{P^{\text{South}}}(\rE, \RL_{P^{\text{South}}}(\rhat))$ and $\mathsf{SubOpt}_{P^{\text{Shifted}}}(\rE, \RL_{P^{\text{Shifted}}}(\rhat))$, respectively. The dots indicate the median and the shaded areas the 0.2 and 0.8 quantiles over the 10 independent realizations.
  \label{fig:experiment}}
\end{figure}

\paragraph{Transferability} We evaluate the transferability of the obtained reward by considering two new environments. First, a south wind setting $P^{\text{South}}$ with wind strength $\beta=1$, and second, a deterministic gridworld $P^{\text{Shifted}}$, with cyclically shifted actions, i.e., Right$\to$ Down, Up$\to$ Right, Left$\to$ Up, Down$\to$ Left. In Figure \ref{fig:experiment}(c) and (d), we illustrate the transferability in terms of $\mathsf{SubOpt}_{P^{\text{South}}}(\rE, \RL_{P^{\text{South}}}(\rhat))$ and $\mathsf{SubOpt}_{P^{\text{Shifted}}}(\rE, \RL_{P^{\text{Shifted}}}(\rhat))$, respectively. We observe that for both environments the transferability improves with a larger second principal angle, thus confirming our theoretical result in Theorem~\ref{thm:global_transferability}. The effect is even more pronounced for the shifted environment. While confirming our results, the experiments also reveal a high sample complexity in terms of expert demonstrations. This is to be expected, as IRL aims to match the expert's empirical occupancy measure, leading to overfitting when there are not enough demonstrations \citep{ho2016generative}. This issue can be mitigated by reducing the dimension of the reward class (see e.g. \citep{abbeel2004apprenticeship}).\looseness-1

\end{document}